\documentclass[nohyperref]{article}

\usepackage{microtype}
\usepackage{graphicx}
\usepackage{booktabs} %

\usepackage{hyperref}
\usepackage{url}

\usepackage[accepted]{icml2023}

\usepackage{amsmath}
\usepackage{amssymb}
\usepackage{mathtools}
\usepackage{amsthm}
\usepackage{amsfonts}

\usepackage[capitalize,noabbrev]{cleveref}
\crefformat{footnote}{#2\footnotemark[#1]#3}

\usepackage{nicefrac}       %
\usepackage{microtype}      %
\usepackage{xcolor}         %

\usepackage{calligra}
\usepackage[T1]{fontenc}
\usepackage{subcaption}
\usepackage{multirow}
\usepackage[normalem]{ulem}
\usepackage{sidecap}
\sidecaptionvpos{figure}{t}

\usepackage{selectp}

\usepackage[textsize=tiny]{todonotes}

\icmltitlerunning{Learning Control by Iterative Inversion}

\newcommand{\state}{s}
\newcommand{\statespace}{S}
\newcommand{\action}{a}
\newcommand{\actionspace}{A}
\newcommand{\traj}{\tau}
\newcommand{\trajstate}{\tau_{\state}}
\newcommand{\trajaction}{\tau_{\action}}
\newcommand{\trajspace}{\Omega}
\newcommand{\trajspacea}{\Omega_\action}
\newcommand{\trajspaces}{\Omega_\state}
\newcommand{\histspace}{H}

\newcommand{\intent}{z}
\newcommand{\intentspace}{Z}
\newcommand{\policy}{\pi}

\newcommand{\prevbuffer}{D_{\mathrm{prev}}}

\newcommand{\steeringbuffer}{D_{\mathrm{steer}}}
\newcommand{\steerratio}{\alpha}
\newcommand{\embed}{\mathcal{Z}}

\newcommand{\methodnamefull}{Iterative Inversion}
\newcommand{\methodname}{IT-IN}

\newcommand{\loss}{\mathcal{L}}
\newcommand{\noise}{\eta}

\newcommand{\F}{\mathcal{F}}

\DeclareMathOperator*{\argmin}{arg\,min}

\newtheorem{definition}{Definition}
\newtheorem{theorem}{Theorem}

\newtheorem{assumption}{Assumption}
\newtheorem{lemma}[theorem]{Lemma}

\newtheorem{remark}{Remark}
\newtheorem{example}{Example}

\newcommand{\alglinelabel}{%
  \addtocounter{ALC@line}{-1}%
  \refstepcounter{ALC@line}%
  \label%
}

\begin{document}

\twocolumn[
\icmltitle{Learning Control by Iterative Inversion}

\icmlsetsymbol{equal}{*}

\begin{icmlauthorlist}
\icmlauthor{Gal Leibovich${^{*}}$}{Intel Labs}
\icmlauthor{Guy Jacob${^{*}}$}{Intel Labs}
\icmlauthor{Or Avner${^{*}}$}{Technion}
\icmlauthor{Gal Novik}{Intel Labs}
\icmlauthor{Aviv Tamar}{Technion}
\end{icmlauthorlist}

\icmlaffiliation{Technion}{Department of Electrical Engineering, Technion, Haifa, Israel}
\icmlaffiliation{Intel Labs}{Intel Labs, Haifa, Israel}

\icmlcorrespondingauthor{Gal Leibovich}{gal.leibovich@intel.com}
\icmlcorrespondingauthor{Guy Jacob}{guy.jacob@intel.com}
\icmlcorrespondingauthor{Or Avner}{or.avner5@gmail.com}
\icmlcorrespondingauthor{Aviv Tamar}{avivt@technion.ac.il}

\icmlkeywords{Reinforcement Learning, Machine Learning, RL, IRL, ICML}

\vskip 0.3in
]

\printAffiliationsAndNotice{Equal contribution, author order determined by coin toss.} %

\begin{abstract}

We propose \textit{iterative inversion} -- an algorithm for learning an inverse function without input-output pairs, but only with samples from the desired output distribution and access to the forward function. The key challenge is a \textit{distribution shift} between the desired outputs and the outputs of an initial random guess, and we prove that iterative inversion can steer the learning correctly, under rather strict conditions on the function. We apply iterative inversion to learn control. Our input is a set of demonstrations of desired behavior, given as video embeddings of trajectories (without actions), and our method iteratively learns to imitate trajectories generated by the current policy, perturbed by random exploration noise. 
Our approach does not require rewards, and only employs supervised learning, which can be easily scaled to use state-of-the-art trajectory embedding techniques and policy representations. Indeed, 
with a VQ-VAE embedding, and a transformer-based policy, we demonstrate non-trivial continuous control on several tasks (videos available at \url{https://sites.google.com/view/iter-inver}). Further, we report an improved performance on imitating diverse behaviors compared to reward based methods.

\end{abstract}

\section{Introduction}\label{s:intro}

The control of dynamical systems is fundamental to various disciplines, such as robotics and automation. Consider the following trajectory tracking problem. Given some deterministic but unknown actuated dynamical system,
\begin{equation}\label{eq:dyn_sys}
    \state_{t+1} = f(\state_t, \action_t),
\end{equation}
where $\state$ is the state, and $\action$ is an actuation, and some reference trajectory, $s_0,\dots,s_T$, we seek actions that drive the system in a similar trajectory to the reference. 

For systems that are `simple' enough, e.g., linear, or low dimensional, classical control theory~\citep{bertsekas1995dynamic} offers principled and well-established system identification and control solutions. However, for several decades, this problem has captured the interest of the machine learning community, where the prospect is scaling up to high-dimensional systems with complex dynamics by exploiting patterns in the system~\citep{mnih2015human,lillicrap2015continuous}.

In reinforcement learning (RL), learning is driven by a manually specified \textit{reward} signal $r(\state,\action)$.
While this paradigm has recently yielded impressive results, defining a reward signal can be difficult for certain tasks, especially when high-dimensional observations such as images are involved.
An alternative to RL is inverse RL (IRL), where a reward is not manually specified. Instead, IRL algorithms \textit{learn} an implicit reward function that, when plugged into an RL algorithm in an inner loop, yields a trajectory similar to the reference. The signal driving IRL algorithms is a \textit{similarity metric between trajectories}, which can be manually defined, or learned~\citep{ho2016generative}. 

We propose a different approach to learning control, which does not require explicit nor implicit reward functions, and also does not require access to a similarity metric between trajectories. Our main idea is that Equation \eqref{eq:dyn_sys} prescribes a mapping $\mathcal{F}$ from an action sequence to a state sequence,
\begin{equation}\label{eq:F_mapping}
    s_0,\dots,s_T = \mathcal{F}(a_0,\dots,a_{T-1}).
\end{equation}
The control learning problem can therefore be framed as finding the \textit{inverse function}, $\mathcal{F}^{-1}$, without knowing $\mathcal{F}$, but with the possibility of evaluating $\mathcal{F}$ on particular action sequences (a.k.a. roll-outs).

Learning the inverse function $\mathcal{F}^{-1}$ using regression can be easy if one has samples of action sequences and corresponding state sequences, and a distance measure over actions. However, in our setting, we do not know the action sequences that correspond to the desired reference trajectories -- a problem that we term \textit{inversion distribution shift}. Interestingly, for some mappings $\mathcal{F}$, an iterative regression technique can be used to find $\mathcal{F}^{-1}$. In this scheme, which we term \textit{\methodnamefull} (\methodname), we start from arbitrary action sequences, collect their corresponding state trajectories, and regress to learn an inverse. We then apply this inverse on the reference trajectories to obtain new action sequences, and repeat. We show that with linear regression, iterative inversion will converge under quite restrictive criteria on $\mathcal{F}$, such as being strictly monotone and with a bounded ratio of derivatives, by establishing a connection between iterative inversion and Newton's method.
Nevertheless, our result shows that for some systems, a controller can be found without a reward function, nor access to a distance measure on states (only actions), and that iterative inversion effectively \textit{steers} the learning to overcome distribution shift.

We then apply iterative inversion to several continuous control problems. In our setting, the desired behavior is expressed through a video embedding of a desired trajectory (without actions), using a VQ-VAE~\citep{van2017neural}, and a deep network policy maps this embedding and a state history to the next action.
The agent generates trajectories from the system using its current policy, \textit{given the desired embeddings as input}, and subsequently learns to imitate its own trajectories, conditioned on their own embeddings. We find that when iterating this procedure, 
the input of the desired trajectories' embeddings
\textit{steers} the learning towards the desired behavior, as in iterative inversion.

Given the strict conditions for convergence of iterative inversion, there is no a-priori reason to expect that our method will work for complex non-linear systems and expressive policies. Curiously, however, we report convergence on all the scenarios we tested, and furthermore, the resulting policy generalized well to imitating trajectories that were not seen in its `steering' training set. 
This surprising observation suggests that \methodname\
may offer a simple supervised learning-based alternative to methods such as RL and IRL, with several potential benefits such as a reward-less formulation, and the simplicity and stability of the (iterated) supervised learning loss function. Furthermore, on experiments where the desired behaviors are abundant and diverse, we report that \methodname\ outperforms reward-based methods, even with an accurate state-based reward.

\section{Iterative Inversion}\label{sec:iterative_inversion}

We describe a general problem of learning an inverse function under a distribution shift, and present the iterative inversion algorithm. We then analyse the convergence of iterative inversion in several simplified settings. In the proceeding, we will apply iterative inversion to learning control. 

Let $\mathcal{F}:\mathcal{X} \to \mathcal{Y}$ be a bijective function. We are given a set of $M$
desired outputs $y_1,\dots,y_M \in \mathcal{Y}$, and an \textit{arbitrary} set of $M$ initial inputs $x_1,\dots,x_M \in \mathcal{X}$. We assume that $\mathcal{F}$ is not known, but we are allowed to observe $\mathcal{F}(x)$ for any $x\in\mathcal{X}$ that we choose during our calculations. Our goal is to find a function $\mathcal{G}:\mathcal{Y} \to \mathcal{X}$ such that for any desired output $y_i$, we have
$
    \mathcal{G}(y_i) = \mathcal{F}^{-1}(y_i).
$

More specifically, we will adopt a parametric setting, and search for a parametric function $\mathcal{G}_\theta$, where $\theta\in \Theta$ is a parameter vector, that minimizes the average loss:
\begin{equation}\label{eq:inversion_problem}
    \min_{\theta\in\Theta} \frac{1}{M}\sum_{i=1}^{M} \mathcal{L}(\mathcal{G}_{\theta}(y_i), \mathcal{F}^{-1}(y_i)).
\end{equation}
For example, $\mathcal{G}_\theta$ could represent the space of linear functions $\mathcal{G}_\theta(y) = \theta^T y + \theta_{bias}$,
and $\mathcal{L}$ could be the squared error between inputs, $\mathcal{L}(x,x') = (x - x')^2$. This example, which is depicted in Figure \ref{fig:problem_illustration} for the 1-dimensional case $\mathcal{X}=\mathcal{Y}=\mathbb{R}$, corresponds to a linear least squares fit of the inverse function. As can be seen, the challenge arises from the \textit{inversion distribution shift} -- the mismatch between the distributions of the desired outputs and initial inputs, 
\begin{definition}[Inversion Distribution Shift]\label{def:dist_shift}
The difference between outputs of the initial distribution $\mathcal{F}(x_1),\dots,\mathcal{F}(x_M)$  and the desired outputs $ y_1,\dots, y_n$.
\end{definition}

\begin{figure}
    \centering
    \includegraphics[width=\columnwidth]{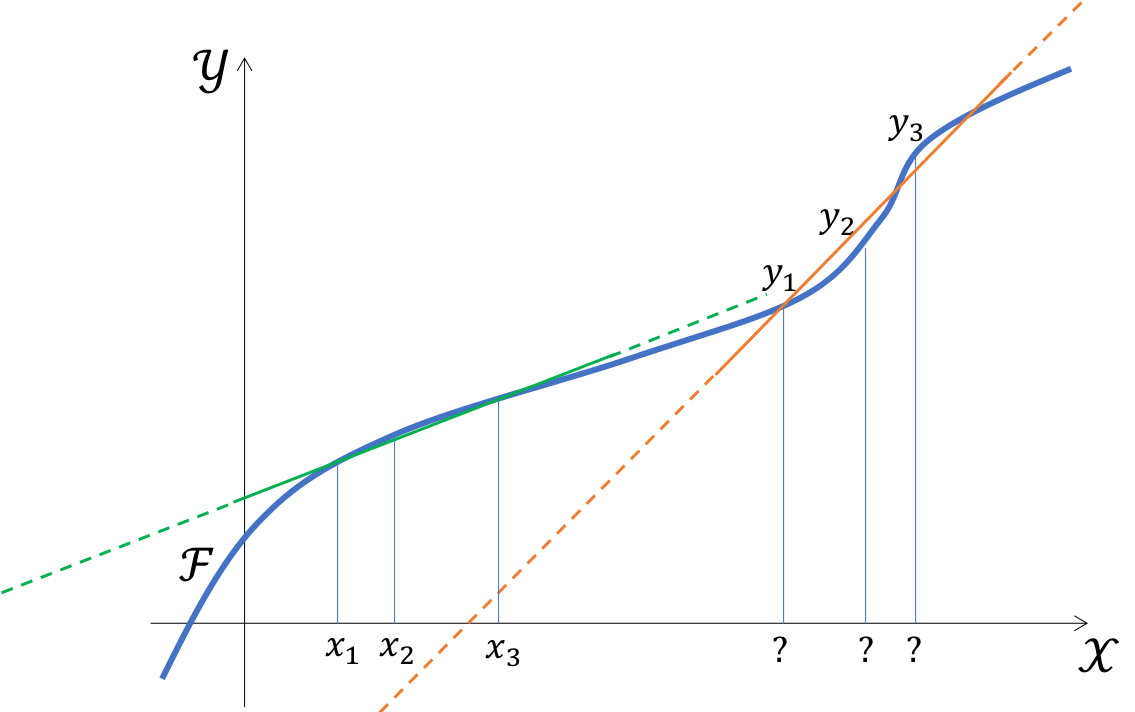}
    \caption{Learning an inverse function under a distribution shift. We wish to learn the inverse function over outputs $y_1,\dots,y_M$, using linear least squares, having matching inputs-outputs for $x_1,\dots,x_M$.}
    \label{fig:problem_illustration}
\end{figure}

The iterative inversion algorithm, proposed in Algorithm \ref{alg:II}, seeks to solve problem \eqref{eq:inversion_problem} iteratively. In the algorithm, and in the proceeding analysis, we define the initial inputs $x_1,\dots,x_M$ implicitly, as the inverse using an initial parameter $\theta_0$, i.e., $x_i = \mathcal{G}_{\theta_0}(y_i)$.
\begin{algorithm}[]
\caption{Iterative Inversion}\label{alg:II}
\begin{algorithmic}[1]
    \REQUIRE Desired outputs $y_1,\dots,y_M \in \mathcal{Y}$, loss function $\mathcal{L}:\mathcal{X}\times \mathcal{X} \to \mathbb{R}$, initial parameter $\theta_0$.
    \FOR{$n=0,1,2,\dots$}
        \STATE Calculate current inputs:\\
        \begin{center}
            $x_1^n,\dots,x_M^n = \mathcal{G}_{\theta_n}(y_1),\dots, \mathcal{G}_{\theta_n}(y_n)$
        \end{center}
        \STATE Calculate current outputs:\\
        \begin{center}
            $y_1^n,\dots,y_M^n = \mathcal{F}(x^n_1),\dots, \mathcal{F}(x^n_M)$
        \end{center}
        \STATE Regression:
        \begin{center}
            $\theta_{n+1} = \argmin_{\theta \in \Theta} \frac{1}{M}\sum_{i=1}^{M} \mathcal{L}(\mathcal{G}_{\theta}(y^n_i), x^n_i)$
        \end{center}
    \ENDFOR
\end{algorithmic}
\end{algorithm}

We next investigate when, and why, should iterative inversion produce an effective solution for \eqref{eq:inversion_problem}. For our analysis, we restrict ourselves to the following setting:
\begin{assumption}\label{ass:linear}
The function class $\mathcal{G}_\theta$ is linear, and the loss $\mathcal{L}$ is the squared error.    
\end{assumption}

We analyze convergence for different classes of functions $\mathcal{F}$.
Denote $X^n\equiv(x^n_1,\dots,x^n_M)^T\in\mathbb{R}^{M \times dim(\mathcal{X})}$, $\F(X^n)\equiv(\F(x_1^n),\dots,\F(x_M^n))^T\in\mathbb{R}^{M \times dim(\mathcal{Y})}$ as the input and output matrices, $\overline{X^n}\equiv\sum_{i=1}^M x_i^n / M\in\mathbb{R}^{dim(\mathcal{X})}$ and $\overline{Y}\equiv\sum_{i=1}^M y_i / M$, $\overline{\F(X^n)}\equiv\sum_{i=1}^M \F(x_i^n) / M\in\mathbb{R}^{dim(\mathcal{Y})}$ as the current inputs, desired outputs, and current outputs means,  $(\cdot)^\dagger$ the Moore-Penrose pseudoinverse operator, and $\F^{-1}$ the ground-truth inverse function.

We start with the simple case of a linear $\F$. As is clear from Figure \ref{fig:problem_illustration}, inverse distribution shift is not a problem in this case, as the inverse function is the same for any $x$,
and iterative inversion converges in a single iteration.

\begin{theorem}\label{the:linear_convergence}
     If Assumption \ref{ass:linear} holds, $\F$ is a linear function and %
    $rank(\F(X^0)-\overline{\F(X^0)})=dim(\mathcal{Y})$ then Algorithm \ref{alg:II} converges in one iteration, i.e.,  $y_1^1,\dots,y_M^1=y_1,\dots,y_M$.
\end{theorem}

We next analyse a non-linear $\F$. Our insight is that iterative inversion can be interpreted as a variant of the classic Newton's method~\citep{ortega2000iterative}, where we replace the unknown Jacobian $J$ of $\F$ with a linear approximation using the current input-output pairs, and the evaluation of $\F$ with the mean of the current outputs.
Recall that Newton's method seeks to find the root $x^*$ of a function $r(x)=\F(x)-y$ using the iterative update rule
$
    x^{n+1}=x^n+(y-\F(x^n))[J(x^n)]^{-1},
$
where $[J(x^n)]^{-1}$ is the Jacobian inverse of $\F$ at $x^n$.
Iterative inversion, similarly, applies the following updating rule, as proved in Appendix \ref{app:update_rule},
\begin{equation}\label{eq:update_rule}
    \overline{X^{n+1}}=\overline{X^n} + \left(\overline{Y}-\overline{\F(X^n)}\right) \Tilde{J}^{-1}_{n} ,
\end{equation}
where $\Tilde{J}^{-1}_{n}\equiv(\F(X^n)-\overline{\F(X^n)})^{\dagger}(X^n-\overline{X^n})$ is the Jacobian of $\mathcal{G}_{\theta_{n+1}}$, the linear regressor plane from $\F(x)$ to $x$ at $x_1^n,\dots,x_M^n$, which can be considered to be an approximation of $[J(\overline{X^n})]^{-1}$. 
When the approximations $\Tilde{J}^{-1}_n\approx   [J(\overline{X^n})]^{-1}$ and $\overline{\F(X^n)}\approx \F(\overline{X^n})$ are accurate, iterative inversion coincides with Newton's method, and enjoys similar convergence properties, as we establish next.

\begin{assumption}\label{ass:f-multidim}
    $\F:\mathbb{R}^K\to\mathbb{R}^K$ is bijective\footnote{\label{foot:bijetion}The bijection assumption is required for our theoretical analysis, to properly define the inverse function. In our experiments, however, we consider intent and action spaces of different dimensions, which are not bijective. In addition, we experiment with environments with non-smooth dynamics, where the bounded derivative assumption does not hold.}, and $\F$ and $\F^{-1}$ are both continuously differentiable.
\end{assumption}
Denote $J(x)$ the Jacobian matrix of $\F$ at $x\in\mathcal{X}$, and $J^{-1}(x)\equiv[J(x)]^{-1}$ the inverse matrix of $J(x)$ and the Jacobian of $\F^{-1}$ at $\F(x)\in\mathcal{Y}$, under Assumption \ref{ass:f-multidim}. Also denote $\|\cdot\|$ to be any induced norm~\citep{matrixanalysisbook}. We assume that the derivatives of $\F$ and $\F^{-1}$ are bounded.
\begin{assumption}\label{ass:d-bounds}
    $\| J(x_1) - J(x_2) \| \leq \gamma$, $\| J(x) \| \leq \zeta$ and $\| J^{-1}(x) \| \leq \beta$ $\forall x_1,x_2,x\in\mathbb{R}^K$.\cref{foot:bijetion}
\end{assumption}
Further assume that at every iteration $n$, the approximations $\Tilde{J}_n^{-1}$ and $\overline{\F(X^n)}$ are accurate enough. 
\begin{assumption}\label{ass:approx}
    $\forall n$: $\| \overline{\F(X^n)}-\F(\overline{X^n})\|\leq \lambda$ and $\Tilde{J}_n^{-1}=J^{-1}(\overline{X^n})(I+\Delta_n)$, $\| \Delta_n \| \leq \delta < 1/{\zeta \beta}$.
\end{assumption}
Assumption \ref{ass:approx} may hold, for example, when the inputs $x_1^n,\dots,x_M^n$ are distributed densely, relative to the curvature of $\F$, and evenly, such that the regression problem in Algorithm \ref{alg:II} is well-conditioned. The requirement $\delta<1/\zeta\beta$ is set to ensure that $\Tilde{J}_n^{-1}$ is non-singular. 

\begin{theorem}\label{the:multidim}
 Suppose Assumptions \ref{ass:linear}, \ref{ass:f-multidim}, \ref{ass:d-bounds} and \ref{ass:approx} hold. Let $\mu\equiv\frac{\zeta^2 \beta \delta}{1-\zeta \beta \delta}$ and assume $\beta(1+\delta)(\gamma + \mu)<1$. Let $\rho \equiv \frac{2\lambda\beta(1+\delta)(\mu + \zeta)}{1-\beta(1+\delta)(\mu+\gamma)}$. Then for every $\epsilon>0$ there exists $k<\infty$ such that $\| \overline{\F(X^k)} - \overline{Y} \| \leq \rho + \epsilon$.
\end{theorem}
Theorem \ref{the:multidim} shows that under sufficient conditions, the iterative inversion method is able to steer learning across any distribution shift, bringing the average output $\overline{\F(X^k)}$ close to the average desired output $\overline{Y}$, regardless of the initial $X^0$. The term $\rho$ can be interpreted as the radius of the ball centered at $\overline{Y}$ that the sequence convergences to.
The proof for Theorem \ref{the:multidim} builds on the analysis of Newton's method to show that \methodname\ is an iterated contraction, and is reported in Section \ref{app:multidim} of the supplementary material. 
To get some intuition about Theorem \ref{the:multidim}, consider the following example.
\begin{example}\label{ex:1d}
Consider the $1$-dimensional case 
$\F:\mathbb{R}\to\mathbb{R}$ (cf. Fig.~\ref{fig:problem_illustration}), where the second approximation in Assumption \ref{ass:approx} are perfect, i.e., $\delta=\mu=0$. Then, the condition for convergence is $\beta(1+\delta)(\gamma+\mu)=\beta \gamma < 1$, which is equivalent to $\frac{\max|\F'(x)|}{\min|\F'(x)|}<2$, i.e., an $\F$ that is `close to linear'. 
\end{example}
The conditions in Theorem \ref{the:multidim} can therefore be intuitively interpreted as $\F$ being `close to linear' globally, and the linear approximation being accurate locally.

In Appendix \ref{app:1d-results}, we provide additional convergence results that use a different analysis technique for the simple case presented in Example \ref{ex:1d}, where $\F:\mathbb{R}\to\mathbb{R}$. These results do not require Assumption \ref{ass:approx}, but still require a condition similar to $\frac{\max|\F'(x)|}{\min|\F'(x)|}<2$, and show a linear convergence rate. We further remark that a quadratic convergence rate is known for Newton's method when the initial iterate is close to optimal; we believe that similar results can be shown for \methodname\ as well. Here, however, we focused on the case of an arbitrary initial iterate, similarly to the experiments we shall describe in the sequel.

\section{Iterative Inversion for Learning Control}

In this section, we apply iterative inversion for learning control. We first present our problem formulation, and then propose an \methodname\ algorithm.

We follow a standard RL formulation. Let $\statespace$ denote the state space,  $\actionspace$ denote the action space, and consider the dynamical system in Equation \ref{eq:dyn_sys}. We assume, for simplicity, that the initial state $s_0$ is fixed, and that the time horizon is $T$.\footnote{A varying horizon can be handled as an additional input to $\mathcal{F}$.}
Given a state-action trajectory $\traj = \state_0,\action_0,\dots,\state_{T-1},\action_{T-1},\state_T\in\trajspace$, where $\trajspace$ denotes the $T$-step trajectory space, we denote by $\trajstate \in \trajspaces$ its state component and by $\trajaction \in \trajspacea$ its action component, i.e., $\trajstate = \state_0,\dots,\state_T$, $\trajaction = \action_0,\dots,\action_{T-1}$, and $\trajspace = \trajspaces \times \trajspacea$.  We will henceforth refer to $\trajstate$ as a state trajectory and to $\trajaction$ as an action trajectory. Let $\mathcal{F}$ denote the mapping from an action trajectory to the resulting state trajectory, as given by Equation \eqref{eq:F_mapping}. 

For presenting our control learning problem, we will assume that $\mathcal{F}$ is bijective, and therefore $\mathcal{F}^{-1}$ is well defined. We emphasize, however, that \textit{our algorithm makes no explicit use of} $\mathcal{F}^{-1}$, and our empirical results are demonstrated on problems where this assumption does not hold.

We represent a state trajectory using an embedding function $\intent = \embed(\trajstate) \in \intentspace$, and we term $\intent$ the \textit{intent}. Note that $z$, by definition, can contain partial information about $\trajstate$, such as the goal state~\citep{ghosh2019learning}. In all the experiments reported in the sequel, we generated intents by feeding a rendered video of the state trajectory into a VQ-VAE encoder, which we found to be simple and well performing. 

Consider a state-action trajectory $\traj$, with a corresponding intent $\embed(\trajstate)$. We would like to learn a policy that reconstructs the intent into its corresponding action trajectory $\trajaction$, and can be used to control the system to produce a similar $\traj$. Let $\histspace_t$ denote the space of $t$-length state-action histories, 
and a policy $\policy_t : \intentspace \times \histspace_t \to \actionspace$. With a slight abuse of notation, we denote by $\policy(\intent) \in \trajspacea$ the action trajectory that is obtained when applying $\pi_t$ sequentially for $T$ time steps (i.e., a rollout).
Similarly to the problem in Section \ref{sec:iterative_inversion}, our goal is to learn a policy such that 
$
    \policy(\embed(\trajstate)) = \mathcal{F}^{-1}(\trajstate).
$
More specifically, let $\loss : \trajspacea \times \trajspacea \to \mathbb{R}$ be a loss function between action trajectories, and let $P(\trajstate)$ denote a distribution over desired state trajectories, we seek a policy $\pi_\theta$ parameterized by $\theta \in \Theta$ that minimizes the average loss:
\begin{equation}\label{eq:control_loss}
    \min_{\theta \in \Theta} \mathbb{E}_{\trajstate \sim P}\left[ \loss\left(\policy_{\theta}(\embed(\trajstate)), \mathcal{F}^{-1}(\trajstate)\right) \right].
    \vspace{-0.5em}
\end{equation}
In our approach we assume that $P(\trajstate)$ is not known, but we are given a set $\steeringbuffer$ of $M$ intents, $\intent_1,\dots, \intent_M$, where $\intent_i = \embed(\trajstate^i)$, and $\trajstate^i$ are drawn i.i.d.~from $P(\trajstate)$. Henceforth, we will refer to $\steeringbuffer$ as the \textit{steering dataset}, as it should act to steer the learning of the inverse mapping towards the desired trajectory distribution $P(\trajstate)$. 

It is worth relating the problem above to the general inverse problem in Section \ref{sec:iterative_inversion}, and the distribution shift in Definition \ref{def:dist_shift}. Initially, the policy is not expected to be able to produce state-action trajectories that match the state trajectories in $\steeringbuffer$, but only trajectories that are output by the initial (typically random) policy. While these initial trajectories could be used as data for imitation learning, yielding an intent-conditioned policy, there is no reason to expect that this policy will be any good for intents in $\steeringbuffer$, which are out-of-distribution with respect to this training data.

We now propose a method for solving Problem \eqref{eq:control_loss} based on iterative inversion, as detailed in Algorithm \ref{alg:ESI}. There are four notable differences from the iterative inversion method in Algorithm \ref{alg:II}. First, we operate on batches of size $N$ instead of on the whole steering data (of size $M$), for computational efficiency. Second, we sample a batch of intents from a mixture of the steering dataset and the intents calculated for rollouts in the previous iteration. We found that this helps stabilize the algorithm. Third, we add random exploration noise to the policy when performing the rollouts, which we found to be necessary (see Sec. \ref{ssec:steering_evaluation}). Fourth, we used a replay buffer for the supervised learning part of the algorithm, also for improved stability.
For $\loss$, we used the MSE between action trajectories,
and for the optimization in line 7, %
we perform several epochs of gradient-based optimization using Adam~\citep{kingma2014adam}, keeping the state history input to $\policy_{\theta}(\hat{\intent})$ fixed as $\trajstate$ when computing the gradient. The size of the replay buffer was set to $K\times N$.

\begin{algorithm}[]
\caption{\methodnamefull\ for Learning Control}
\label{alg:ESI}
\begin{algorithmic}[1]
    \REQUIRE Steering data $\steeringbuffer$, exploration noise parameter $\noise$, steering ratio $\steerratio\in [0,1]$, batch size $N$
    \STATE Initialize $\prevbuffer = \steeringbuffer$, $\theta_0$ arbitrary
    \FOR{ $n=0,1,2,\dots$}
        \STATE Sample $\steerratio N$ intents from $\steeringbuffer$ and $(1-\steerratio) N$ intents from $\prevbuffer$, yielding $z^1,\dots,z^N$ 
        \STATE Perform $N$ rollouts $\traj^1,\dots,\traj^N$ using policy $\policy_{\theta_n}$ with input intents $z^1,\dots,z^N$, adding exploration noise $\noise$
        \STATE \alglinelabel{line:relabel} Compute intents for the rollouts $\hat{\intent}^i = \embed (\trajstate^i)$, $i\in1,\dots,N$ 
        \STATE Add intents, trajectories $\{\hat{\intent}^i, \traj^i\}$ to Replay Buffer
        \STATE \alglinelabel{line:sl} Train $\policy_{\theta_{n+1}}$ by supervised learning: $\theta_{n+1} = \argmin_{\theta \in \Theta} \sum_{\{\hat{\intent},\traj\}\in \textrm{Replay Buffer}}\left[ \loss\left(\policy_{\theta}(\hat{\intent}), \trajaction\right) \right]$ 
        \STATE Set $\prevbuffer = \left\{ \hat{\intent}^i \right\}_{i=1}^{N} $
    \ENDFOR
\end{algorithmic}
\end{algorithm}

The astute reader may notice that while the desired intents $\steeringbuffer$ are used as input to the data collection policy, they are actually \textit{detached} from the training loss throughout learning, by the relabelling in line 5. %
What, then, drives the learning closer to $\steeringbuffer$? Our analysis in Section \ref{sec:iterative_inversion} explains how, under suitable conditions of $\mathcal{F}$, the policy can in fact be steered appropriately by such an algorithm.

\begin{remark}
Note that a small loss over the actions, per Eq.~\eqref{eq:control_loss}, does not necessarily imply a similar trajectory to the reference. This is since (1) errors accumulate over time (generally, the error in the states can grow linearly with $T$; cf. Assumption~\ref{ass:d-bounds}), and (2) a possible distribution shift, which grows with $T$, between the states used to train the policy (Line 7) %
and states that would be encountered during a rollout from the trained policy when given the same intent, and can lead to state errors that grow quadratically with $T$~\cite{ross2011reduction}. This is a different distribution shift from our Definition \ref{def:dist_shift}, which concerns the difference between the rollouts of the initial policy and the desired outcomes and holds also for $T=1$. While steering is not expected to fix error accumulation, in practice, we found that Algorithm \ref{alg:ESI} did lead to accurate trajectories, even for systems where a small mistake can be dramatic. We attribute this to the training of a policy to imitate many noisy trajectories, which stabilizes the learned control.
\end{remark}

Note the simplicity of the \methodname\ algorithm -- it only involves exploration and supervised learning; there are no rewards, and the loss function is routine. In Section \ref{sec:experiments_top}, we provide empirical evidence that, perhaps surprisingly -- given the strict conditions for convergence of iterative inversion -- \methodname\ yields well-performing policies on nontrivial tasks. 

\section{Related Work}

Iterative inversion is similar to the \textit{breeder} algorithm~\citep{nair2008analysis} for inverting black-box functions, but with one crucial difference -- breeder starts from an \textit{in distribution} input-output pair, termed a \textit{prototype code vector}, and grows input-output pairs around it, to \textit{avoid} distribution shift. Our method does not require \textit{any} in distribution samples, and  we show how distribution shift can be overcome by steering.

Goal conditioned supervised learning (GCSL,~\citealt{ghosh2019learning}) is essentially a special case of iterative inversion where intents are chosen to be goal observations. Our contribution, however, is investigating the \textit{steering} component of the method -- an idea that originated in the algorithm of \citet{ghosh2019learning}, but was not investigated at all to our knowledge. \citet{ghosh2019learning}'s theory assumes no distribution shift (coverage of all goals in the data), and the idea that steering can overcome distribution shift is, to the best of our knowledge, novel. In addition, we show that our approach can learn to track trajectories accurately, which is important for tasks where the \textit{whole trajectory is important}, and not just the goal.

Learning inverse dynamics models is popular in robotics~\citep{nguyen2010using,calandra2015learning,meier2016towards,christiano2016transfer}, and typically requires the full state trajectory for predicting corresponding actions, and training data \textit{from the desired state-action distribution}; our approach requires an embedding of the desired trajectory, and our focus is on the setting of inversion distribution shift.
The only study we are aware of in this direction is \citet{hong2020adversarial}, where an RL agent is trained to steer data collection to areas where the inverse model errs. However, \citet{hong2020adversarial} require the full state trajectory as policy input, require RL training as part of their method, and report difficulty in scaling to high dimensional action spaces, where their curiosity-based exploration is not effective enough to steer learning towards desired behavior. 
Recently, \citet{baker2022video} used a transformer to learn an inverse model conditioned on video, but they collected human-labelled data to train their model on desired behavior trajectories. Our approach is self-supervised.

In learning from demonstrations~\citep{argall2009survey}, data typically contains both states and actions, enabling direct supervised learning, either by behavioral cloning~\citep{pomerleau1988alvinn} or interactive methods such as DAgger~\citep{ross2011reduction}.
Inverse RL (IRL) can learn from demonstrations without actions, and methods such as apprenticeship learning~\citep{abbeel2004apprenticeship} or generative adversarial imitation learning~\citep{ho2016generative,peng2022ase}
simultaneously train a critic that discriminates between data trajectories and policy rollouts (a classification problem), and a policy that confuses the critic as best as possible (an RL problem). 
Our emphasis is on learning a policy that can reconstruct \textit{diverse} behaviors, conditioned on an intent -- different from most IRL studies that consider a \textit{single task}~~\citep{ho2016generative,edwards2019imitating}, or unconditional generation~\cite{peng2022ase}.
While \citet{fu2019language,ding2019goal} considered a goal-conditioned IRL setting, we are not aware of IRL methods that can be conditioned on a more expressive description than a target goal state, such as a complete trajectory embedding, as we explore here. Our approach also avoids the need for training a critic or an RL agent in an inner loop.

In self-supervised RL, the agent does not receive a reward and uses its own experience to explore the environment by training a goal-conditioned policy and proposing novel goals~\citep{pathak2018zero,ecoffet2019go,hazan2019provably,sekar2020planning,mendonca2021discovering}. The space of all trajectories is much larger than the space of all states, and we are not aware of methods that demonstrably explore such a space. For this reason, in our approach we steer the exploration towards a set of desired trajectories.

\section{Experiments}
\label{sec:experiments_top}

\begin{figure*}
    \centering
    \includegraphics[width=\textwidth,keepaspectratio]{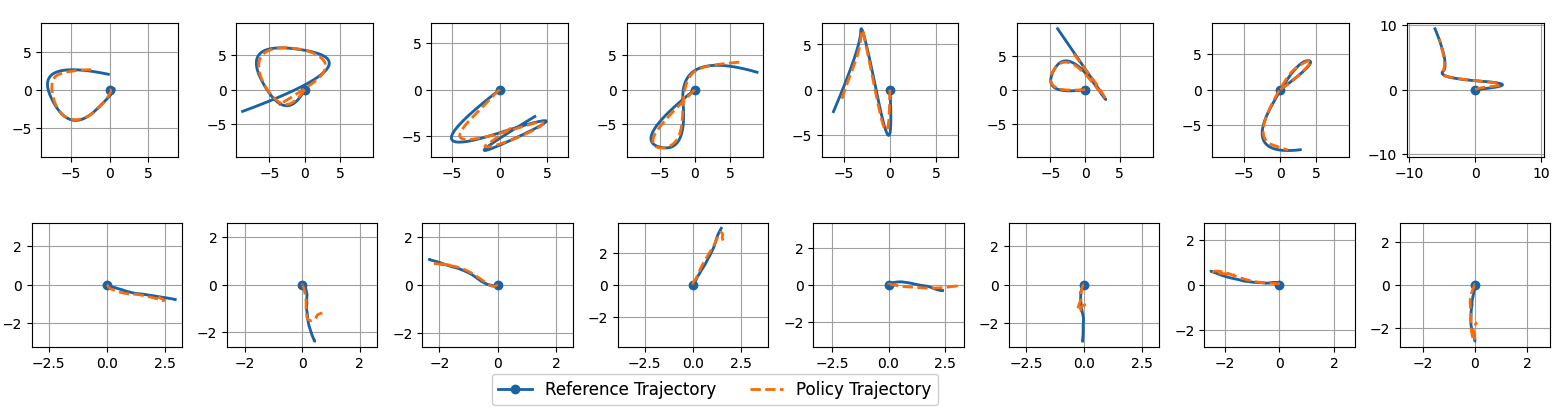}
    \vspace{-1em}
    \caption{Particle results on \texttt{Splines} (top) and \texttt{Deceleration} (bottom). Here $T=64$ and $|\steeringbuffer|=500$. All trajectories start at \texttt{(0,0)}, marked by a blue circle. In \texttt{Deceleration}, the particle quickly decelerates to a stop at $t=32$ -- note the small overshoot at the end of each reconstructed trajectory, due to imperfect reconstruction of stopping in place.}
    \label{fig:particle_all}
\end{figure*}
\begin{figure*}
    \centering
    \vspace{-1em}
    \includegraphics[width=\textwidth,keepaspectratio]{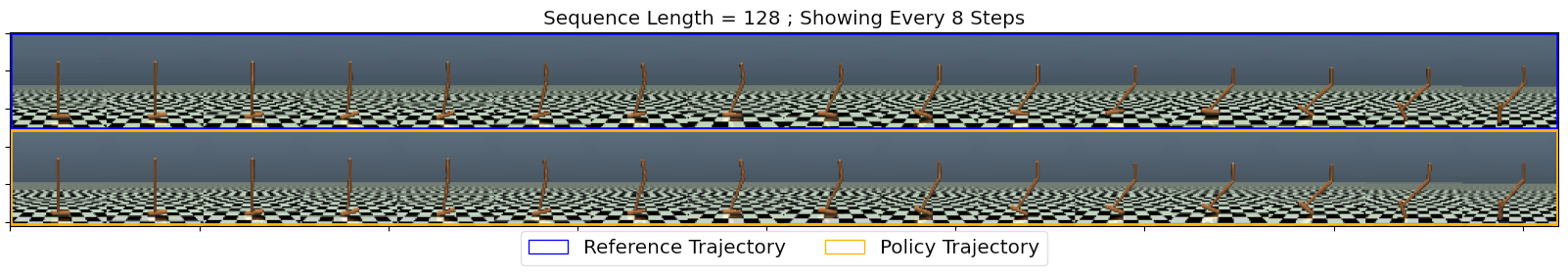}
    \vspace{-2em}
    \caption{
    Trajectory reconstructions in \texttt{Hopper-v2}, with $T=128$ and $|\steeringbuffer|=500$. Additional rollouts are presented in Appendix~\ref{app:results_hopper} and in the supporting video results. 
    }
    \label{fig:hopper_128}
    \vspace{-1em}
\end{figure*}

In this section, we evaluate \methodname\ on several domains. Our investigation is aimed at studying the unique features of \methodname\, and especially, the \textit{steering} behavior that we expect to observe.
We start by describing our evaluation domains, and implementation details that are common to all our experiments. We then present a series of experiments aimed at answering specific questions about \methodname. To appreciate the learned behavior, we encourage the reader to view our supporting video results at the project website \footnote{\label{foot:website}\url{https://sites.google.com/view/iter-inver}}.

\subsection*{Common Settings:} 
\textbf{VQ-VAE Intents:} For all our experiments, we generate intents using a VQ-VAE embedding of a rendered video of the trajectory. Rendering settings are provided next for each environment. We use VideoGPT's VQ-VAE implementation~\citep{yan2021videogpt}. An input video of size $64 \times 64 \times T$ (w, h, t) is encoded into a $16 \times 16 \times T/4$ integer intent ${\intent}^i$ given a codebook of size $50$. Each integer represents a float vector of length $4$. The training of the VQ-VAE is not the focus of this work, and we detail the training data for each VQ-VAE separately for each domain in the supplementary material. We remark that by visually inspecting the reconstruction quality, we found that our VQ-VAEs generalized well to the trajectories seen during learning.

\textbf{GPT-based policies and exploration noise} The policy architecture is adapted from VideoGPT \citep{yan2021videogpt}, and consists of 8 layers, 4 heads and a hidden dimension of size 64. The model is conditioned on the intent via cross-attention. In the supplementary material, we report similar results with a GRU-based policy. Our exploration noise adds a Gaussian noise of scale $\noise$ to the action output.

\textbf{Evaluation Protocol}: 
While our algorithm only uses a loss on actions, a loss on the resulting trajectories is often easier to interpret for measuring performance. We measure the sum of Euclidean distances between agent state variables, accumulated over time, as a proxy for trajectory similarity; in our results, this measure is denoted as MSE. 
Except when explicitly noted otherwise, all our results are evaluated on test trajectories (and corresponding intents) that were not in the steering data, but were generated from the same trajectory distribution. None of the trajectories we plot or our video results are cherry picked.

\subsection*{Domains}
\textbf{2D Particle:}
\label{ssec:domains_particle}
A particle robot is moved on a friction-less 2D plane, by applying a force $F=[F_X,F_Y]$ for a duration of $\Delta t$.
The observation space includes particle positions and velocities $S=[X,Y,V_X,V_Y]$, and motion videos are rendered using Matplotlib Animation~\citep{Hunter:2007}. 

While relatively simple for control, this environment allows for distinct and diverse behaviors that are easy to visualize. We experiment with 2 behavior classes, for which we procedurally created training trajectories: (1) \texttt{Spline} motion, and (2) \texttt{Deceleration} motion. Both require highly coordinated actions, and are very different from the motion that a randomly initialized policy induces. Full details about the datasets are described in Appendix~\ref{particle:datasets}. 

\textbf{Reacher:}
A 2-DoF robotic arm from OpenAI Gym's Mujoco Reacher-v2 environment~\citep{1606.01540}. While usually in Reacher-v2 the agent is rewarded for reaching a randomly generated target, the goal in our setting is for the policy to reconstruct the whole arm motion, as given by the intent. The intent is encoded from a video of the motion rendered using Mujoco~\citep{todorov2012mujoco}. We handcrafted a trajectory dataset, termed \texttt{FixedJoint}, which is fully described in Appendix~\ref{reacher:datasets}.

\textbf{Hopper:}
From OpenAI Gym's Mujoco Hopper-v2 environment~\citep{1606.01540}. The dataset is from D4RL's \texttt{hopper-medium-v2}~\citep{fu2020d4rl}, and consists of mostly forward hopping behaviors (see Appendix~\ref{hopper:datasets}). There are several challenges in this domain: (1) the dynamics are non-linear, and include a non-smooth contact with the ground; (2) the desired behavior (hopping) is very different from the behavior of an untrained policy (falling), and requires applying a very specific force exactly when making contact with the ground (a `bottleneck' in state space); and (3) the camera is fixed on the agent, and forward movement can only be inferred from the movement of the background.

\subsection*{Steering Evaluation}
\label{ssec:steering_evaluation}

The first question we investigate is whether \methodname\ indeed steers learning towards the desired behavior. To answer this, we consider domains where the desired behavior is \textit{very different} from the behavior of the initial random policy -- the \texttt{Spline} and \texttt{Deceleration} motions for the particle, and the hopping behavior for \texttt{Hopper-v2}. As we show in Figure~\ref{fig:particle_all} (for particle), 
and Figure~\ref{fig:hopper_128} (for \texttt{Hopper-v2}), \methodname\ produces a policy that can track the desired behavior with high accuracy. Videos of rollouts in the \texttt{Hopper-v2} environment\cref{foot:website} demonstrate that the policy is able to \textit{accurately reconstruct different hopping motions}, based only on their encoded intents. To the best of our knowledge, such accurate motion control, with only an action-less video encoding as input, on a non-trivial dynamical system has not been demonstrated before.
We further show, in Figure~\ref{fig:particle_splines_multi_lengths_grid} and Figure~\ref{fig:splines_multi_length_strips} in the supplementary material, that \methodname\ works well for different trajectory lengths $T$.

Another question is whether \methodname\  really steers the policy towards the desired trajectories, or improves general properties of the policy, allowing a generally better reconstruction. We explore this question by a \textit{cross-evaluation} -- evaluating the performance of a policy trained with steering intents from \texttt{Particle:Splines} on test intents from \texttt{Particle:Deceleration}, which we will refer to as out-of-distribution intents, and vice versa. Interestingly, as Table~\ref{table:steering}
shows, performance on out-of-distribution intents is significantly worse than the performance that would have been obtained by training the policy with these intents as the steering dataset, and is even worse or comparable to training with no steering at all (cf. Table~\ref{table:MSEs}). Example rollouts from this experiment are shown in \ref{app:test_on_ood}.

\subsection*{Evaluation of Exploration Noise}
We also evaluated the importance of the exploration noise.
We tested \texttt{Splines} with $T=64$ and \texttt{Hopper-v2} with $T=128$ with and without exploration noise, and a large $\steeringbuffer$ (1740 for Hopper-v2, 500 for Particle).
As the results in Table~\ref{table:test_exploration} in the supplementary material show, exploration noise $\noise$ is crucial for the training procedure to converge towards the desired behavior. 
We believe that exploration improves the conditioning of the supervised learning problem, and helps produce a policy that stabilizes control along the desired trajectory.

\begin{table}[]
    \centering
    \caption{Steering cross-evaluation. See Appendix~\ref{app:test_on_ood} for corresponding trajectory visualizations. In all cases $|\steeringbuffer|=500$.}
    \label{table:steering}
    \begin{tabular}{ccc}
        \toprule
        Test Dataset & Steering Dataset & MSE \\
        \midrule
        \multirow{2}{*}{\texttt{Splines}}      & \texttt{Splines}      & 69.2  \\
                                               & \texttt{Deceleration} & 210.9 \\
        \midrule
        \multirow{2}{*}{\texttt{Deceleration}} & \texttt{Splines}      & 28.1  \\
                                               & \texttt{Deceleration} & 18.5  \\
        \bottomrule
    \end{tabular}
\end{table}

\subsection*{Steering Dataset Size and Generalization}
\label{ssec:steering_size}
We next evaluate the generalization performance of \methodname\ to intents that were not seen in the data, but correspond to state trajectories drawn from $P(\trajstate)$. To investigate this, we consider a domain where the desired behavior is \textit{very diverse} -- the \texttt{Spline} motions for the particle. 
We also report results on domains where the behavior is less diverse, such as \texttt{Hopper-v2} and \texttt{Deceleration} motions for particle.
Naturally, we expect generalization to correlate with $M$, the size of $\steeringbuffer$. As our results in Table~\ref{table:MSEs} show, additional steering data indeed improves generalization to unseen trajectories, albeit with diminishing returns as the amount of steering data is increased. As expected, in the more diverse distribution there was more gain to reap from additional data (significant improvement up to $|\steeringbuffer|=50$), compared with the less diverse domain (most of the improvement is achieved already with $|\steeringbuffer|=10$).
Trajectory visualizations for \texttt{Splines} with different sizes of $\steeringbuffer$ are shown in Appendix~\ref{app:results_splines_steering_size}.

\begin{table*}[]
    \caption{
    Steering Dataset Size and Generalization. Here $T=64$, and we show MSE averaged over 3 random seeds.
    Note that $|\steeringbuffer|=0$ represents the case where no steering is used at all. In this case, we use trajectories sampled from a random policy to initialize $|\prevbuffer|$ (see Algorithm~\ref{alg:ESI}). (*)~For \texttt{Hopper-v2}, the maximal $|\steeringbuffer|$ is $1740$ due to a limited amount of data in D4RL.  
    }
    \vspace{-0.5em}
    \label{table:MSEs}
    \resizebox{\textwidth}{!}{%
    \begin{tabular}{lcccccc}
        \toprule
        \multicolumn{1}{c}{}  & $|\steeringbuffer|$ = 0 & $|\steeringbuffer|$ = 10 & $|\steeringbuffer|$ = 50 & $|\steeringbuffer|$ = 100 & $|\steeringbuffer|$ = 500 & $|\steeringbuffer|$ = 2000 $^*$\\
        \midrule
        \texttt{Particle:Splines}       & 199.7 & 105.4 & 75.8 & 72.7 & 69.2 & 66.9 \\
        \texttt{Particle:Deceleration}  & 30.0 & 20.5 & 21.1 & 20.2 & 17.9 & 18.6 \\
        \texttt{Hopper-v2}          & 173.3 & 68.2 & 67.2 & 64.9 & 67.0 & 63.0 \\
        \bottomrule
    \end{tabular}%
    }
\end{table*}

\subsection*{Comparison with RL and IRL Baselines}
\label{ssec:rl_baselines}

In essence, \methodname\ performs imitation learning from observations. Thus, a natural comparison is with methods such as GAIL from observations~\cite{torabi2018generative}. Previous IRL literature focused on learning only a single task, and measured accuracy as the task success (e.g., the success of hopper hopping). However, the setting where \methodname\ shines, and which we focus our evaluation on, is where the policy must be able to reconstruct a \textit{diverse} set of behaviors (from their corresponding intents). Thus, our evaluation is not whether a single task succeeds, but how accurately any desired trajectory is tracked (e.g., hopping in a very specific motion). The multi-task setting, where the agent needs to perform many different behaviors specified by different intents, has not been studied to the best of our knowledge.

We compare \methodname\ with GAIL from observations~\cite{torabi2018generative} in the \texttt{Particle:Splines} environment. As we consider the multi-task setting, we also add comparison with RL baselines where we manually set a relevant reward function. These RL baselines serve to show that any advantage of \methodname\ over IRL is not due to some implementation detail (both use the same model architecture), but due to the difficulty of its RL component to learn in this multi-intent setting. 
We consider two reward functions tailored to the \texttt{Particle:Splines} environment: (1) \texttt{STATE-MSE}: MSE between desired position and current position, and (2) \texttt{INTENT-MSE}: a sparse reward that is the MSE between the intents of the desired trajectory and the executed trajectory, given at the end of the episode. \texttt{STATE-MSE} is privileged compared to \methodname\ and is arguably stronger than any IRL method in this task, as the reward is dense, and exactly captures the desired behavior.
Any IRL method will run RL in an inner loop, with a reward that is less precise. \texttt{INTENT-MSE} is motivated by the fact that \methodname\ effectively learns some similarity measure in intent space, and this reward captures this idea explicitly.

We used exactly the same policy architecture for all comparisons. We found that both RL and IRL methods did not train well with the GPT-based policy architecture\footnote{Difficulty of RL with transformers was discussed in \citep{parisotto2020stabilizing,hausknecht2022consistent}.}, therefore we report results for the GRU policy (also for \methodname), which is described in detail in Appendix~\ref{app:gru}.
We used PPO~\citep{schulman2017proximal} for RL training, based on the implementation of \citet{pytorchrl}. Our GAIL implementation is also based on \citet{pytorchrl}, with modifications to follow the setup of \citet{torabi2018generative} and to add the intent as a context input. Additional details are in Appendices~\ref{app:rl} and~\ref{app:gail}.

\begin{figure}[]
    \centering
    \includegraphics[width=\columnwidth,keepaspectratio]{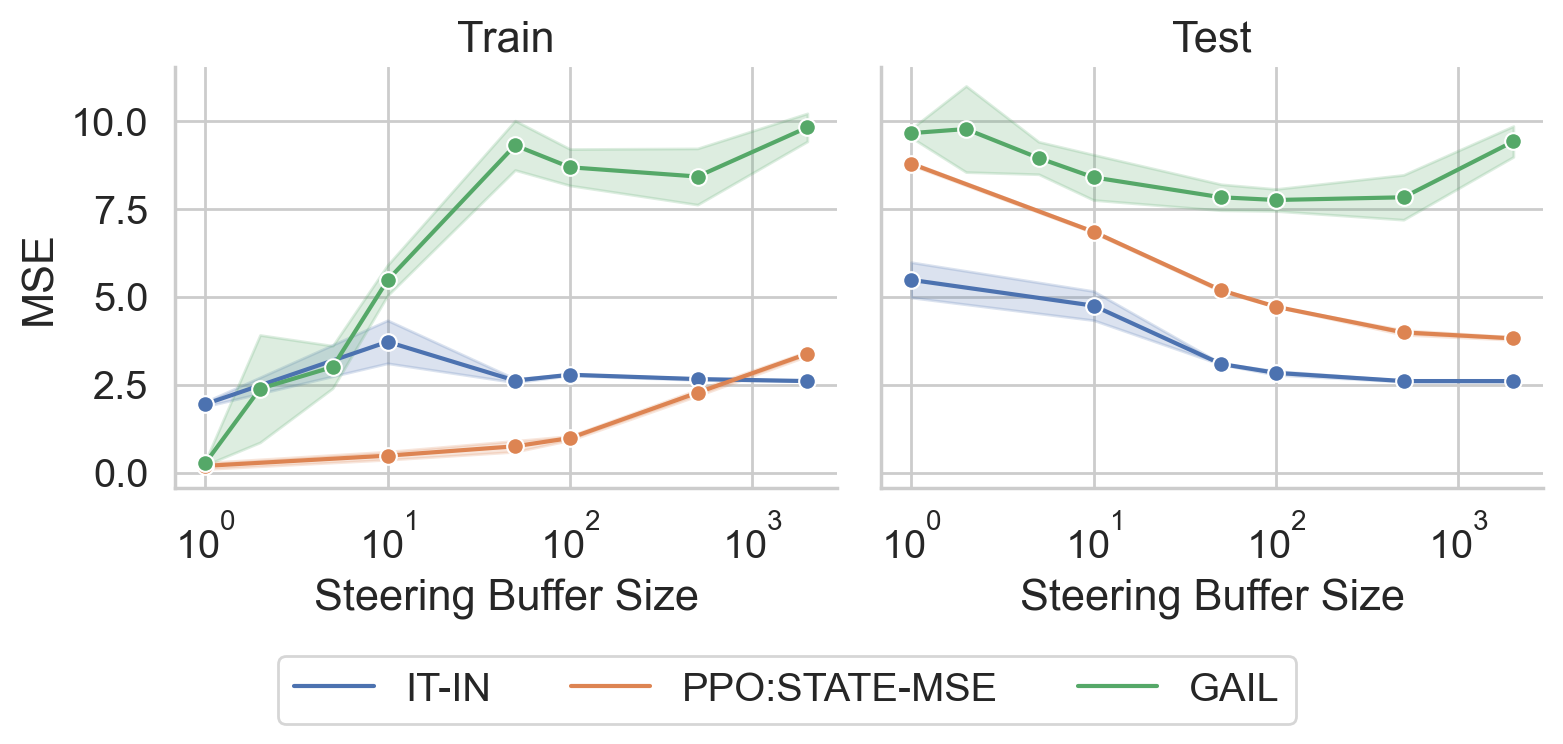}
    \vspace{-1em}
    \caption{Comparison of \methodname\ with PPO and GAIL baselines, in the \texttt{Particle:Splines} environment with $T=16$. All results are MSE (lower is better), each represented with a mean and standard deviation of 3 random seeds. Note that \methodname\ outperforms baselines on test trajectories (graph on the right) for all $\steeringbuffer$ sizes. For GAIL, $\steeringbuffer$ is also used as the expert data for the discriminator. Additional results are shown in Figure~\ref{fig:rl_appendix}. }
    \label{fig:rl}
    \vspace{-0.5em}
\end{figure}

In Figure~\ref{fig:rl}, we report results both on a held-out test set of trajectories, and on the training trajectories. As expected, for a small $\steeringbuffer$, \texttt{STATE-MSE} obtains near perfect reconstruction of training trajectories, yet high error on test trajectories, as the precise reward makes it easy for PPO to overfit. Interestingly, however, when increasing the size of $\steeringbuffer$, it becomes more difficult to overfit with PPO, even with the \texttt{STATE-MSE} reward. This highlights a difficulty of RL in the multi-task setting: note that for $|\steeringbuffer|=2000$, the performance of \texttt{STATE-MSE} on training is worse than the performance of \methodname\ on test! 
Our results suggest that vanilla PPO is not well suited to training policies conditioned on very diverse contexts (Our context here is a vector of length 4096). We mention that the recent related work of \citet{peng2022ase} trained context embeddings together with RL, which may explain their success in learning diverse skills. 
Importantly, on test data, \methodname\ significantly outperforms both RL methods for all $\steeringbuffer$ sizes, even though \methodname\ \textit{does not} use the privileged information in the reward. We attribute this finding to the combination of stable supervised learning updates, and not relying on a reward. 
Our results for \texttt{INTENT-MSE} do not come close to \methodname, which we attribute to the more difficult learning from sparse reward (see Appendix~\ref{app:rl_irl_baselines}). 
Finally, GAIL was outperformed by both \texttt{STATE-MSE} RL and \methodname\ (except for the single demonstration case -- the standard "single task" GAIL setup). As expected, GAIL was not able to find a better reward than \texttt{STATE-MSE} here.
Additional baseline-comparison results are presented in Appendix~\ref{app:rl_irl_baselines}.

\section{Discussion}
We presented a new formulation for learning control, based on an inverse problem approach, and demonstrated its application to learning deep neural network policies that can reconstruct diverse behaviors, given an embedding of the desired trajectory. 
We developed the fundamental theory underlying iterative inversion, and demonstrated promising results on several simple tasks. We also found that for very diverse behaviors, our formulation learns and generalizes more effectively than RL or IRL approaches, which we attribute to the stable supervised-learning method at its core.

We only considered a particular trajectory embedding based on an off-the-shelf VQ-VAE, which we found to be general and practical. Important questions for future work include characterizing the effect of the embedding on performance, and training an embedding jointly with the policy. Additionally, the exploration noise, which we found to be important, can potentially be replaced with more advanced exploration strategies from the RL literature.

Another question is how to generate intents from a partial description of a trajectory, such as a natural language description. Diffusion models, which have recently gained popularity for learning distributions over latent variables~\citep{rombach2021highresolution}, are one potential approach for this.

Remaining open questions include the gap between the strict conditions for convergence under a linear approximation in our theory and the stable performance we observed in practice with expressive policies and non-linear dynamics, and whether iterative inversion can be extended to non-deterministic systems. Our work provides the fundamentals for further investigating these important questions.

\section*{Acknowledgements}
This work received funding from the European Union (ERC, Bayes-RL, Project Number 101041250). Views and opinions expressed are however those of the author(s) only and do not necessarily reflect those of the European Union or the European Research Council Executive Agency. Neither the European Union nor the granting authority can be held responsible for them. 

\bibliography{sample}
\bibliographystyle{icml2023}

\newpage
\appendix
\onecolumn

\section{Proofs}
\subsection{Proof of Equation \ref{eq:update_rule}}\label{app:update_rule}
Throughout this and the rest of the theoretical proofs, with a slight abuse of notation, when a vector $u\in\mathbb{R}^N$ is added to a matrix $A\in\mathbb{R}^{M\times N}$, the addition is row-wise: $A+u\equiv A+\mathbf{1}u$ where $\mathbf{1}u=(u,\dots,u)^T\in\mathbb{R}^{M\times N}$.\\

We remind the reader of the notations defined in Section \ref{sec:iterative_inversion}. Denote $X^n\equiv(x^n_1,\dots,x^n_M)^T\in\mathbb{R}^{M \times dim(\mathcal{X})}$, $\F(X^n)\equiv(\F(x_1^n),\dots,\F(x_M^n))^T\in\mathbb{R}^{M \times dim(\mathcal{Y})}$ as the input and output matrices, $\overline{X^n}\equiv\sum_{i=1}^M x_i^n / M\in\mathbb{R}^{dim(\mathcal{X})}$ and $\overline{Y}\equiv\sum_{i=1}^M y_i / M$, $\overline{\F(X^n)}\equiv\sum_{i=1}^M \F(x_i^n) / M\in\mathbb{R}^{dim(\mathcal{Y})}$ as the current inputs, desired outputs, and current outputs means,  and $(\cdot)^\dagger$ the Moore-Penrose pseudoinverse operator. We also define $Y=(y_1,\dots,y_M)^T\in\mathbb{R}^{M\times dim(\mathcal{Y})}$. The approximated linear function is $\mathcal{G}_{\Theta,b}(Y)=Y\Theta + b$ where $\Theta\in\mathbb{R}^{dim(\mathcal{Y})\times dim(\mathcal{X})}$ and $b\in\mathbb{R}^{1 \times dim(\mathcal{X})}$ (note that we explicitly add the bias parameter $b$). At iteration $n$:
\begin{align*}
    \Theta_{n+1},b_{n+1}=\argmin_{\Theta, b} \| \mathcal{F}(X^n) \Theta + b - X^n \| ^2 .
\end{align*}
This is an ordinary linear least squares problem with the solution \citep[Section 3.11.1]{linearregbook2003}:
\begin{equation}\label{eq:theta_b}
\begin{aligned}
 \Theta_{n+1} &= (\mathcal{F}(X^n)-\overline{\mathcal{F}(X^n)})^{\dagger}(X^n-\overline{X^n}), \qquad\qquad&  b_{n+1} &= \overline{X^n}-\overline{\mathcal{F}(X^n)}\Theta_{n+1}.
\end{aligned}
\end{equation}
Then,
\begin{equation*}
X^{n+1} = \mathcal{G}_{\Theta_{n+1},b_{n+1}}(Y)=Y\Theta_{n+1}+b_{n+1}=\overline{X^n}+(Y-\overline{\mathcal{F}(X^n)})\Theta_{n+1},
\end{equation*}
and averaging over the points yields the result:
\begin{equation*}
\overline{X^{n+1}} =\overline{X^n}+(\overline{Y}-\overline{\mathcal{F}(X^n)})\Theta_{n+1}.
\end{equation*}

\subsection{Proof of Theorem \ref{the:linear_convergence}}
Using the notation defined in Appendix~\ref{app:update_rule}.\\ Assuming $\mathcal{F}(X)= X F+h$ is a linear function with $F\in\mathbb{R}^{dim(\mathcal{X})\times dim(\mathcal{Y})}$ and $h\in\mathbb{R}^{1 \times dim(\mathcal{Y})}$. \\
Assuming $rank(\mathcal{F}(X^0)-\overline{\mathcal{F}(X^0)})=dim(\mathcal{Y})$ then $(\mathcal{F}(X^0)-\overline{\mathcal{F}(X^0)})^T(\mathcal{F}(X^0)-\overline{\mathcal{F}(X^0)})$ is invertible and $\Theta_1$ defined on Equation \ref{eq:theta_b} is well defined. \\
Then using the fact that $\mathcal{F}(X^n)-\overline{\mathcal{F}(X^n)}=(X^n-\overline{X^n})F$:
\begin{align*}
    \Theta_1 = (\mathcal{F}(X^0)-\overline{\mathcal{F}(X^0)})^{\dagger}(X^0-\overline{X^0})=\left((X^0-\overline{X^0})F\right)^{\dagger}(X^0-\overline{X^0}),
\end{align*}
and satisfies
\begin{align*}
    \Theta_1 F = I,
\end{align*}
where $I$ is the identity matrix. The bias term, according to Equation \ref{eq:theta_b}, is 
\begin{align*}
    b_1=\overline{X^0}-\overline{\mathcal{F}(X^n)}\Theta_1=\overline{X^0}-(\overline{X^0}F+h)\Theta_1.
\end{align*}
At the end of iteration 1, $X^1=Y\Theta_1+b_1$ and its matching outputs equal to the desired outputs:
\begin{align*}
    \mathcal{F}(X^1) &=X^1 F + h\\
    &=Y\Theta_1 F + b_1 F + h\\
    &=Y + (\overline{X^0}-(\overline{X^0}F+h)\Theta_1) F + h\\
    &=Y + (\overline{X^0}F-\overline{X^0}F-h) + h\\
    &= Y.
\end{align*}

\subsection{Proof of Theorem \ref{the:multidim}}\label{app:multidim}

For clarity in our presentation, we will use the following notation: $J^{-1}_n\equiv J^{-1}(X^n)$, $J_n\equiv J(\overline{X^n})$, $\Tilde{\F}_n\equiv\overline{\F(X^n)}$ and $\F_n \equiv \F(\overline{X^n})$. We also define $H_n\equiv\F_n-\overline{Y}$ and $\Tilde{H}_n\equiv\Tilde{\F}_n-\overline{Y}$ 
 
First we show that $\Tilde{J}_n^{-1}$ is non-singular. Since $\delta < \frac{1}{\zeta\beta}$ then $\rho(J_n \Delta_n J_n^{-1}) \leq \| J_n \Delta_n J_n^{-1} \| \leq \delta \zeta \beta < 1$ where $\rho(A)$ denotes the spectral radius of $A$. The first inequality is proven in \citet[Thm. 5.6.9]{matrixanalysisbook} and the second inequality is a result of the sub-multiplicative property of the induced norm, $\|J_n \Delta_n J_n^{-1}\| \leq \|J_n\| \|\Delta_n\| \|J_n^{-1}\|$. Therefore $(I+J_n \Delta_n J_n^{-1})$ is non-singular, and $\Tilde{J}_n^{-1}=J_n^{-1}(I + J_n \Delta_n J_n^{-1})$ is non-singular as a multiplication of non-singular matrices. 

We denote by $\Tilde{J}_n\equiv\left(\Tilde{J}_n^{-1}\right)^{-1}$ its matrix inverse, and obtain the following bounds:
\begin{equation}\label{eq:h-bound}
 \| H_n - \Tilde{H}_n \| = \| \F_n - \Tilde{\F}_n \| \leq \lambda
\end{equation}
\begin{equation}\label{eq:j-1-approx-bound}
    \| \Tilde{J}^{-1}_n \| = \| (I + \Delta_n) J_n^{-1}  \| \overset{(*)}{\leq} \| J^{-1}_n \| (1 + \| \Delta_n \|) \leq \ \beta(1+\delta)
\end{equation}
\begin{equation}\label{eq:j-diff-bound}
    \| \Tilde{J}_n - J_n \| \overset{(**)}{\leq} \frac{\|J_n\|^2 \| \Delta_n J_n^{-1}\|}{1-\|J_n \Delta_n J_n^{-1} \|}
    \leq  \frac{\|J_n\|^2 \| \Delta_n \| \| J_n^{-1}\|}{1-\|J_n \| \| \Delta_n \| \| J_n^{-1} \|} \leq \frac{\zeta^2 \delta \beta}{1-\zeta \delta \beta} \equiv \mu
\end{equation}
\begin{equation}
    \| \Tilde{J}_n \| \leq \| \Tilde{J}_n - J_n \| + \| J_n \| \leq \mu + \zeta
\end{equation}
\begin{equation}\label{eq:x-diff-bound}
    \| \Tilde{\F}_n - \overline{Y} \| = \| \Tilde{H}_n \| = \| (\overline{X^{n+1}} - \overline{X^n})\Tilde{J}_n  \| \leq \| \Tilde{J}_n \| \| \overline{X^{n+1}} - \overline{X^n}\|  \leq (\mu + \zeta)\| \overline{X^{n+1}} - \overline{X^n}\|
\end{equation}
Inequality $(*)$ is due to the sub-multiplicative and sub-additive properties of the induced norm, $\| (I + \Delta_n) J_n^{-1}  \| \leq \|J_n^{-1}\| (\|I\| + \|\Delta_n\|)$ with $\|I\|=1$. Inequality $(**)$ is developed in \citet[p.~381]{matrixanalysisbook}, in the context of bounding the error in the inverse of an error-perturbed matrix. Also note that the rest of the inequalities in \eqref{eq:j-diff-bound} are well defined since $\delta<1/\zeta\beta$.

The proof now continues similarly to the proof of \citet[12.3.3]{ortega2000iterative}. We set $G_n=\overline{X^n}-\Tilde{H}_n \Tilde{J}^{-1}_n =\overline{X^{n+1}}$, and show that $G_n$ is an Iterated Contraction:

\begin{align*}
    \| \overline{X^{n+2}} - \overline{X^{n+1}} \| &= \| \overline{X^{n+1}} - \Tilde{H}_{n+1} \Tilde{J}_{n+1}^{-1} - \overline{X^{n+1}}\| = \| \Tilde{H}_{n+1} \Tilde{J}^{-1}_{n+1} \| \overset{(1)}{\leq}  \beta(1+\delta) \| \Tilde{H}_{n+1} \| \\
    &\leq \beta(1+\delta) \| \Tilde{H}_{n+1} - \Tilde{H}_n - (\overline{X^{n+1}} - \overline{X^n})\Tilde{J}_n  \| \\
    &\overset{(2)}{\leq} \beta(1+\delta) \| \Tilde{H}_{n+1} - \Tilde{H}_n - (\overline{X^{n+1}} - \overline{X^n})J_n  \| + \beta(1+\delta) \| \Tilde{J}_n - J_n \| \| \overline{X^{n+1}} - \overline{X^n} \|\\
    &\overset{(3)}{\leq} \beta(1+\delta) \left(2\lambda + \| H_{n+1} - H_n - (\overline{X^{n+1}} - \overline{X^n})J_n  \|\right) + \beta(1+\delta) \mu \| \overline{X^{n+1}} - \overline{X^n} \|\\
    &= \beta(1+\delta) \left(2\lambda + \| \F(\overline{X^{n+1}}) - \F(\overline{X^n}) - (\overline{X^{n+1}} - \overline{X^n})J(\overline{X^n})  \| + \mu \| \overline{X^{n+1}} - \overline{X^n} \|\right)\\
    &\overset{(4)}{\leq} \beta(1+\delta) \left(2\lambda + \gamma \| \overline{X^{n+1}} - \overline{X^n} \| + \mu \| \overline{X^{n+1}} - \overline{X^n} \|\right)\\
    &\leq \beta(1+\delta) \left(\frac{2\lambda}{\| \overline{X^{n+1}} - \overline{X^n} \|} + \gamma + \mu \right)\| \overline{X^{n+1}} - \overline{X^n} \|\\
    &\overset{(5)}{\leq} \beta(1+\delta) \left(\frac{2\lambda(\mu + \zeta)}{\| \Tilde{\F}_n - \overline{Y} \|} + \gamma + \mu \right)\| \overline{X^{n+1}} - \overline{X^n} \|
\end{align*}
where inequality $(1)$ holds because of Bound \ref{eq:j-1-approx-bound}, $(2)$ is the triangle inequality, $(3)$ is due to the Bounds \ref{eq:h-bound} and \ref{eq:j-diff-bound} and the triangle inequality. Inequality $(4)$ is proven in \citet[3.2.12]{ortega2000iterative}, {using the assumption $\forall x_1,x_2: \|J(x_1)-J(x_2)\|\leq \gamma$}, and inequality $(5)$ is from Bound \ref{eq:x-diff-bound}.  

Denote the function  $g:\mathbb{R}\to\mathbb{R}$, $g(x)\equiv \beta(1+\delta) \left(2\lambda(\mu + \zeta)/x + \gamma + \mu \right)$. Then
\begin{align*}
    \| \overline{X^{n+2}} - \overline{X^{n+1}} \| \leq g(\| \Tilde{\F}_n - \overline{Y} \|) \| \overline{X^{n+1}} - \overline{X^n} \|
\end{align*}
Assuming $\beta(1+\delta)(\gamma+\mu)<1$:
\begin{align*}
    g(\| \Tilde{\F}_n - \overline{Y} \|) =1 \iff \| \Tilde{\F}_n - \overline{Y} \|=\frac{2\lambda\beta(1+\delta)(\mu + \zeta)}{1-\beta(1+\delta)(\mu+\gamma)}\equiv\rho
\end{align*}
$g$ is strictly-decreasing function, thus if $\| \Tilde{\F}_n - \overline{Y} \| \geq \rho +\epsilon$ for some $\epsilon>0$ then $g(\| \Tilde{\F}_n - \overline{Y} \|)\leq\alpha<1$, where $\alpha$ is independent of $\| \Tilde{\F}_n - \overline{Y} \|$. Then, as long as $\| \Tilde{\F}_{n-1} - \overline{Y} \| \geq \rho +\epsilon$:
\begin{align*}
    \| \Tilde{\F}_n - \overline{Y} \| \leq (\mu + \zeta)\| \overline{X^{n+1}} - \overline{X^n}\| \leq \dots \leq (\mu + \zeta)\alpha^n \| \overline{X^1}-\overline{X^0} \|
\end{align*}
where the first inequality holds due to \ref{eq:x-diff-bound}.
Then, for every $\epsilon>0$ there exists $k<\infty$ such that one of the following holds:
\begin{enumerate}
    \item There exists $n< k$ where $\| \Tilde{\F}_n - \overline{Y} \| < \rho +\epsilon$ and the proof is done.
    \item For all $n<k$: $\| \Tilde{\F}_n - \overline{Y} \| \geq \rho +\epsilon$ and $\alpha^k \leq \frac{\rho+\epsilon}{(\mu + \zeta)\| \overline{X^1}-\overline{X^0} \|}$. Then $\| \Tilde{\F}_k - \overline{Y} \| \leq (\mu + \zeta)\alpha^k \|\overline{X^1}-\overline{X^0} \| \leq \rho + \epsilon$ and the proof is done.
\end{enumerate}

\subsection{Convergence Results for 1-Dimensional $\F$}\label{app:1d-results}

We restrict ourselves to the 1-dimensional case, where $\mathcal{X}=\mathcal{Y}=\mathbb{R}$, and assume the function $\F$ is strictly monotone and its maximum and minimum slopes are not too different, thus the function is "close to" linear. We then show convergence at a linear rate.

Let $\mathcal{S}^\F(x_1,x_2)\equiv(\F(x_1)-\F(x_2))/(x_1-x_2)$ denote the slope of $\F$ between $x_1$ and $x_2$, and $\max |\mathcal{S}^\F|\equiv \max_{x_1,x_2\in\mathcal{X}} |\mathcal{S}^\F(x_1,x_2)|$ denote the maximum absolute slope of $\F$ and similarly $\min |\mathcal{S}^\F|\equiv \min_{x_1,x_2\in\mathcal{X}} |\mathcal{S}^\F(x_1,x_2)|$ the minimum absolute slope.
\begin{assumption}\label{ass:F}
 $\F$ is continuous and strictly monotone, and $\frac{\max |\mathcal{S}^\F|}{\min |\mathcal{S}^\F|}\leq 2-\epsilon$ for some $0<\epsilon\leq 1$.
\end{assumption}
\begin{theorem}\label{the:2p}
    Assume $\mathcal{X}=\mathcal{Y}=\mathbb{R}$, that Assumption \ref{ass:F} holds, and that there are only two desired outputs $M=2$. Then for any $i\in \{1,2\}$ and any iteration $n$: $|\F(x_i^{n+1})-y_i|\leq(1-\epsilon)|\F(x_i^{n})-y_i|$.
\end{theorem}
When the number of desired outputs is greater than 2, then convergence for each output is generally not guaranteed.
\begin{theorem}\label{the:1d-multipoints}
        Assume $\mathcal{X}=\mathcal{Y}=\mathbb{R}$, that Assumption \ref{ass:F} holds and that at iteration $n$, $\forall i$ $x_i^n<\F^{-1}(\overline{Y})$ or $\forall i$ $x_i^n>\F^{-1}(\overline{Y})$ . Then $\left|\overline{X^{n+1}}-\F^{-1}(\overline{Y})\right|\leq(1-\epsilon)\left|\overline{X^n}-\F^{-1}(\overline{Y})\right|$.
\end{theorem}
Theorem \ref{the:1d-multipoints} guarantees that after a finite number of iterations, the output segment intersects with the desired output segment. Note that Theorems \ref{the:2p} and \ref{the:1d-multipoints} do not require any kind of approximations as in Assumption \ref{ass:approx}, nor for $\F$ to be differentiable. 

\subsubsection{Proof of Theorem \ref{the:2p}}\label{app:2d}
Denote $S_{max}^\F \equiv \max_{x_1,x_2} \mathcal{S}^\F (x_1,x_2)$ and similarly $S_{min}^\F \equiv \min_{x_1,x_2} \mathcal{S}^\F (x_1,x_2)$.\\

Assuming $\mathcal{X}=\mathcal{Y}=\mathbb{R}$. Then the approximated linear function is $\mathcal{G}_{\theta,b}(y)=y\theta + b$ where $\theta,b\in\mathbb{R}$ are scalars.\\
At iteration $n+1$ and for $i\in[1,M]$:
\begin{equation}\label{eq:x_i-1d}
    x_i^{n+1} = \mathcal{G}_{\theta_{n+1},b_{n+1}}(y_i)=y_i\theta_{n+1}+b_{n+1},
\end{equation}
\begin{align*}
    \theta_{n+1}, b_{n+1} = \argmin_{\theta, b} \sum_{i=1}^M \left(\theta \F(x_i^n)+b - x_i^n\right)^2.
\end{align*}

\begin{lemma}\label{lm:theta_bounds_1d}
if $\mathcal{X}=\mathcal{Y}=\mathbb{R}$ then $\forall n$: $\frac{1}{S_{max}^\F}\leq \theta_{n+1}  \leq \frac{1}{S_{min}^\F}$ if $\F$ is strictly increasing and $\frac{1}{\mathcal{S}_{min}^\F}\leq \theta_{n+1}  \leq \frac{1}{\mathcal{S}_{max}^\F}$ if $\F$ is strictly decreasing. 
\end{lemma}
\begin{proof}
We will prove for strictly increasing $\F$. The proof for strictly decreasing $\F$ is symmetrical.\\ 
Without loss of generality, we assume that $X^n$ is sorted: $\forall i$: $x_i^n\leq x_{i+1}^n$. Let $k>i$ then:
\begin{align*}
    \F(x_i^n) + \mathcal{S}_{min}^\F (x_k^n-x_i^n) \leq \F(x_k^n)\leq   \F(x_i^n) + \mathcal{S}_{max}^\F (x_k^n-x_i^n),
\end{align*}
\begin{align*}
    \frac{1}{\mathcal{S}_{max}^\F}(\F(x_k^n)-\F(x_i^n))\leq x_k^n - x_i^n\leq \frac{1}{\mathcal{S}_{min}^\F}(\F(x_k^n)-\F(x_i^n)),
\end{align*}
\begin{align*}
    \theta_{n+1} &=\frac{\frac{1}{M}\sum_{i=1}^M (x_i^n-\overline{X^n})\left(\F(x_i^n)-\overline{\F(X^n)}\right)}{\frac{1}{M}\sum_{i=1}^M \left(\F(x_i^n)-\overline{\F(X^n)}\right)^2}=\\
    &=\frac{\frac{1}{M^2}\sum_{i=1}^{M-1}\sum_{k=i+1}^M (x_k^n-x_i^n)\left(\F(x_k^n)-\F(x_i^n)\right)}{\frac{1}{M^2}\sum_{i=1}^{M-1}\sum_{k=i+1}^M \left(\F(x_k^n)-\F(x_i^n)\right)^2}\\
    &\leq \frac{\frac{1}{M^2}\sum_{i=1}^{M-1}\sum_{k=i+1}^M \frac{1}{\mathcal{S}_{min}^\F} \left(\F(x_k^n)-\F(x_i^n)\right)^2}{\frac{1}{M^2}\sum_{i=1}^{M-1}\sum_{k=i+1}^M \left(\F(x_k^n)-\F(x_i^n)\right)^2}=\frac{1}{S_{min}^\F},
\end{align*}
and
\begin{align*}
    \theta_{n+1} \geq \frac{\frac{1}{M^2}\sum_{i=1}^{M-1}\sum_{k=i+1}^M \frac{1}{\mathcal{S}_{max}^\F} \left(\F(x_k^n)-\F(x_i^n)\right)^2}{\frac{1}{M^2}\sum_{i=1}^{M-1}\sum_{k=i+1}^M \left(\F(x_k^n)-\F(x_i^n)\right)^2}=\frac{1}{\mathcal{S}_{max}^\F}.
\end{align*}
\end{proof}

When $M=2$, the regression line passes exactly at the points $(\mathcal{F}(x^n_1), x^n_1)$ and $(\mathcal{F}(x^n_2), x^n_2)$, and $b_{n+1}$ also takes the following forms:
\begin{align*}
    b_{n+1}=x^n_1-\theta_{n+1}\F(x^n_1)=x^n_2-\theta_{n+1}\F(x^n_2).
\end{align*}
Then, plugging $b_{n+1}$ in Equation \ref{eq:x_i-1d} we get for every $i\in[1,2]$:
\begin{align*}
    x_i^{n+1}=x_i^n+\theta_{n+1}(y_i-\F(x_i^n)).
\end{align*}
Denote the slope of $\F$ between $x_i^{n+1}$ and $x_i^n$: $\mathcal{S}^\F(x^{n+1}_i,x_i^n)\equiv\frac{\F(x_i^{n+1})-\F(x_i^n)}{x_i^{n+1}-x_i^n}=\frac{\F(x_i^{n+1})-\F(x_i^n)}{\theta_{n+1}(y_i-\F(x_i^n))}.
$
Then the following equations hold:
\begin{align*}
    \F(x_i^{n+1})=\F(x_i^n)+\theta_{n+1}\mathcal{S}^\F(x_i^{n+1},x_i^n) (y_i-f(x_i^n)),
\end{align*}
\begin{equation}\label{eq:y_diff_1d}
    y_i - \F(x_i^{n+1})=\left(1-\theta_{n+1}\mathcal{S}^\F(x_i^{n+1},x_i^n) \right)(y_i-\F(x_i^n)).
\end{equation}
Using Lemma \ref{lm:theta_bounds_1d}, and since $\F$ is always increasing or always decreasing, then $\theta_{n+1} \mathcal{S}^\F(x_i^{n+1},x_i^n))>0$ and
\begin{align*}
   \frac{1}{2-\epsilon} \leq \frac{\min |\mathcal{S}^\F|}{\max |\mathcal{S}^\F|}\leq \theta_{n+1} \mathcal{S}^\F(x_i^{n+1},x_i^n)) \leq \frac{\max |\mathcal{S}^\F|}{\min |\mathcal{S}^\F|}\leq 2-\epsilon ,
\end{align*}
\begin{equation}\label{eq:gamma_bound}
    \left| 1- \theta_{n+1} \mathcal{S}^\F(x_i^{n+1},x_i^n)) \right| \leq \max \left\{ |1-\frac{1}{2-\epsilon}|, |1-\epsilon |\right\}=1-\epsilon.
\end{equation}
Then, plugging into Equation \ref{eq:y_diff_1d},
\begin{align*}
        \left| y_i - \F(x_i^{n+1}) \right| =  \left|1-\theta_{n+1}\mathcal{S}^\F(x_i^{n+1},x_i^n) \right| \left|y_i-\F(x_i^n)\right| \leq (1-\epsilon) \left| y_i - \F(x_i^n) \right|.
\end{align*}
Note the convergence in one iteration for the linear case when $\epsilon=1$.

\subsubsection{Proof of Theorem \ref{the:1d-multipoints}}\label{app:1d-multipoints}
Denote $L_n$:
\begin{align*}
    L_n&\equiv\frac{\overline{Y}-\overline{\F(X^n)}}{\F^{-1}(\overline{Y})-\overline{X^n}}=\frac{\frac{1}{M}\sum_{i=1}^N \overline{Y}-f(x_i^n)}{\frac{1}{M}\sum_{k=1}^N \F^{-1}(\overline{Y})-x_k^n}=\sum_{i=1}^M \left( \frac{\F^{-1}(\overline{Y})-x_i^n}{\sum_{k=1}^M \F^{-1}(\overline{Y})-x_k^n}\right) \frac{\overline{Y}-f(x_i^n)}{\F^{-1}(\overline{Y})-x_i^n}\\
    &= \sum_{i=1}^M w_{n,i} \frac{\overline{Y}-f(x_i^n)}{\F^{-1}(\overline{Y})-x_i^n}=\sum_{i=1}^M w_{n,i}\ \mathcal{S}^\F(\F^{-1}(\overline{Y}),\ x_i^n).
\end{align*}
Where $w_{n,i}=\frac{\F^{-1}(\overline{Y})-x_i^n}{\sum_{k=1}^M \F^{-1}(\overline{Y})-x_k^n}$, $\sum_{i=1}^M w_{n,i}=1$ and, since we assumed $\forall i\ x_i^n<\F^{-1}(\overline{Y})$ or that $\forall i\ x_i^n>\F^{-1}(\overline{Y})$, then $\forall i$ $w_{n,i}>0$. Therefore $L_n$ is a weighted-mean of the slopes and  
$\mathcal{S}_{min}^\F \leq L_n \leq \mathcal{S}_{max}^\F$.\\
From Equation \ref{eq:update_rule} the following holds: 
\begin{align*}
    \overline{X^{n+1}}-\overline{X^n}=\theta_{n+1}(\overline{Y}-\overline{\F(X^n)})=\theta_{n+1}L_j(\F^{-1}(\overline{Y})-\overline{X^n}),
\end{align*}
\begin{equation}
    \F^{-1}(\overline{Y}) - \overline{X^{n+1}} = (1-\theta_{n+1}L_j)(\F^{-1}(\overline{Y})-\overline{X^n}).
\end{equation}
Using Lemma \ref{lm:theta_bounds_1d} and the inequalities $\mathcal{S}_{min}^\F \leq L_n \leq \mathcal{S}_{max}^\F$, Inequality \eqref{eq:gamma_bound} from Appendix \ref{app:2d} also applies for $L_n$, and we obtain:
\begin{align*}
        \left| 1- \theta_{n+1} L_n \right| \leq 1-\epsilon,
\end{align*}
\begin{align*}
    \left| \F^{-1}(\overline{Y}) - \overline{X^{n+1}} \right| \leq (1-\epsilon)\left| \F^{-1}(\overline{Y})- \overline{X^n} \right|.
\end{align*}
\subsection{Tightness of the Derivative Ratio Bound for 1-Dimensional $\F$}
We consider the case where $\F$ is a 1-dimensional function and provide a simple negative example to demonstrate the tightness of the derivative ratio bound, $\frac{\max |\mathcal{S}^\F|}{\min |\mathcal{S}^\F|}< 2$. As described in Example \ref{ex:1d}, in this case, the second approximation in Assumption \ref{ass:approx} is perfect, and using the bounds in Assumption \ref{ass:d-bounds}, the condition for convergence in Theorem \ref{the:multidim} is equivalent to $\frac{\max |\mathcal{S}^\F|}{\min |\mathcal{S}^\F|}< 2$. We assume this condition is not satisfied, i.e., that the maximum slope is more than twice the minimum slope, and show that convergence does not occur, and that the first initial input guess is closer to the desired inputs than all the following iterations. 
\begin{example}\label{ex:negative}
    Let $\epsilon,a,b,\delta,\Delta>0$ and define the continuous 1-dimensional, increasing and piece-wise linear function $\F$ with 5 linear segments:
    \begin{align*}
        \F(x)=\begin{cases}
            x & x\leq a\\
            a + (2+\epsilon)(x-a) & a< x \leq a+\Delta \\
            (1+\epsilon)\Delta + x & a + \Delta < x \leq a + \Delta + b\\
            a + b + (2+\epsilon)(x-a-b) & a + b + \Delta < x \leq a + b + 2\Delta\\
            2(1+\epsilon)\Delta + x  & a + b + 2\Delta < x 
        \end{cases}
    \end{align*}
    Notice that the minimum slope of $\F$ is $1$ and the maximum is $(2+\epsilon)$.
    We set $\Delta\equiv\frac{b/2+3\delta}{\epsilon}$. \\
    Let there be two desired inputs $x_1^*=a+\Delta+ b/2-\delta$, $x_2^*=a+\Delta+ b/2 + \delta$, and their outputs $y_1^*=\F(x_1^*)=a+b/2-\delta+(2+\epsilon)\Delta=a+b+2\Delta+2\delta$, $y_2^*=\F(x_2^*)=a+b+2\Delta+4\delta$. They are both placed in the 3rd linear segment of the function.\\
    We set the initial inputs $x_1^0=a-\delta$, $x_2^0=a$ (placed in the 1st linear segment), their corresponding outputs will be the same: $y_1^0=x_1^0, y_2^0=x_2^0$, and $\mathcal{G}_{\theta_1}(y)=y$.\\
    At iteration $n=1$, the inputs will be $x_1^1=\mathcal{G}_{\theta_1}(y_1^*)=y_1^*, x_2^1=\mathcal{G}_{\theta_1}(y_2^*)=y_2^*$ (in the 5th linear segment) and $\mathcal{G}_{\theta_2}(y)=y-2(1+\epsilon)\Delta$. \\
    At iteration $n=2$, the inputs will be $x_1^2=\mathcal{G}_{\theta_2}(y_1^*)=y_1^*-2(1+\epsilon)\Delta=a-4\delta, x_2^3=\mathcal{G}_{\theta_1}(y_2^*)=a-2\delta$ (in the 1st linear segment), and $\mathcal{G}_{\theta_3}(y)=y$ again. \\
    Similarly, at every even iteration, $x_1^{2n}=x_1^2, x_2^{2n}=x_2^2$ the current input-output pairs will be equal and in the 1st segment, and at every odd iteration $x_1^{2n+1}=x_1^1, x_2^{2n+1}=x_2^1$ the current input-output pairs will be equal and in the 5th segment. Also, starting from the 2nd iteration, the distance from the current input-outputs to the desired inputs-outputs will stay constant, which is larger than the distance from the initial input-output pairs: $\forall n>1$: $\| x_1^* - x_1^{n} | = \| x_2^* - x_2^{n} \| = \Delta + b/2 + 3\delta > \Delta + b/2 = \| x_1^* - x_1^0 \| = \| x_2^* - x_2^0 \|$.
\end{example}

\begin{figure}[]
    \centering
    \includegraphics[width=0.8\linewidth,keepaspectratio]{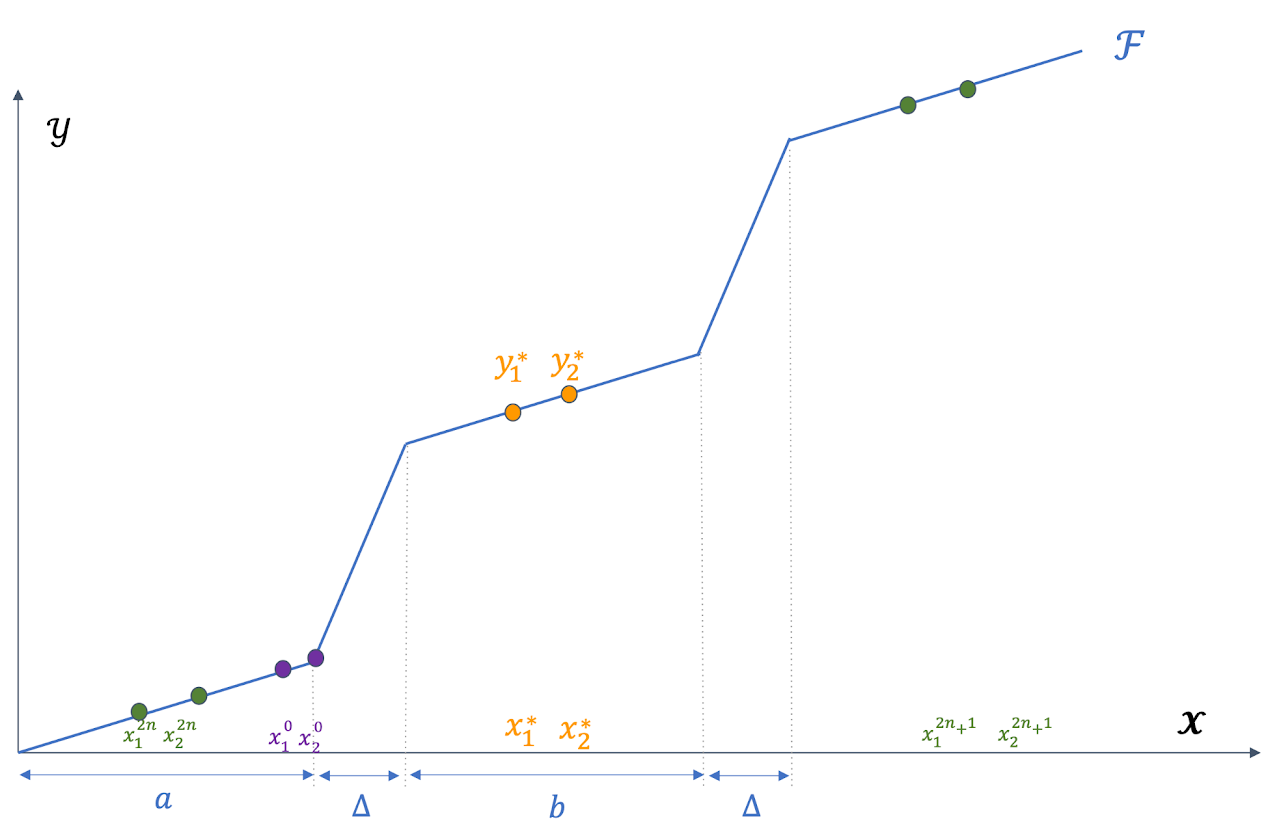}
    \vspace{-1em}
    \caption{Visual illustration of the 1-dimensional function $\F$, described in Example \ref{ex:negative}, for which convergence is not achieved. }
    \label{fig:negative_example}
    \vspace{-0.5em}
\end{figure}

\newpage

\section{Experimental Details}
Table~\ref{table:common_hps} contains a list of common hyperparameter values that we have used for all the experiments. Table~\ref{table:particle_reacher_hps} contains \texttt{Particle} and \texttt{Reacher-v2} specific hyperparameters, while Table~\ref{table:hopper_hps} is listing \texttt{Hopper-v2} specific hyperparameters. We note that the minor difference in hyperparameter values between the domains evaluated is purposed only at achieving slightly better MSE results per domain. We observed that the steering behavior was relatively robust to hyperparameter values.

\begin{table}[h]
    \centering
    \caption{Common hyperparameters for all experiments }
    \label{table:common_hps}
        \begin{tabular}{l|l}
        \hline
        \multicolumn{1}{c|}{Hyperparameter}                    & \multicolumn{1}{c}{Value} \\ \hline
        Learning rate                                          & 5e-4                      \\
        Sampled rollouts per epoch ($N$)                       & 200 x Minibatch size      \\
        Training iterations                                    & 2,000                     \\
        Training buffer size [rollouts] ($K$ x $N$)            & 40 x N                    \\
        Steering buffer $\steeringbuffer$ size [rollouts]      & 500                       \\
        Ratio of steering intents in minibatch ($\steerratio$) & 0.3                       \\
        Gradient norm clipping                                 & 0.5                       \\
        GPT: \# layers                                         & 8                         \\
        GPT: \# heads                                          & 4                         \\
        GPT: hidden layer size                                 & 64                        \\
        GPT: dropout                                           & 0.2                       \\
        GPT: attention dropout                                 & 0.3                       \\
        \end{tabular}
\end{table}

\begin{table}[h]
    \begin{minipage}[c]{0.5\textwidth}
    \centering
    \caption{Particle \& Reacher-v2 hyperparameters}
    \label{table:particle_reacher_hps}
        \begin{tabular}{l|l}
        \hline
        \multicolumn{1}{c|}{Hyperparameter} & \multicolumn{1}{c}{Value} \\ \hline
        Minibatch size (rollouts)           & 8                         \\
        Noise scale ($\noise$)              & 4.0                       \\
        Total epochs ($n$)                  & 160                      
        \end{tabular}
    \end{minipage}
    \begin{minipage}[c]{0.5\textwidth}
    \centering
    \caption{Mujoco Hopper-v2 hyperparameters}
    \label{table:hopper_hps}
        \begin{tabular}{l|l}
        \hline
        \multicolumn{1}{c|}{Hyperparameter} & \multicolumn{1}{c}{Value} \\ \hline
        Minibatch size (rollouts)           & 6                         \\
        Noise scale ($\noise$)              & 1.0                       \\
        Total epochs ($n$)                  & 130                      
        \end{tabular}
    \end{minipage}
\end{table}

\subsection{Particle Robot}
The 2D plane in which the robot is allowed to move is a finite square, with the maximum coordinates (denoted $C_{max}$) increasing for longer horizons. When rendering the videos we include the entire 2D plane, up to the maximum coordinates. When evaluating policies, a validation set of 2,000 trajectories was used, which were unseen during training of the policies.

\subsubsection{Datasets}
\label{particle:datasets}
\paragraph{\texttt{Splines}}
Trajectories follow the function of a B-spline curve\footnote{\url{https://en.wikipedia.org/wiki/B-spline}}. The curves are of degree 2 with 5 control points, which are uniformly sampled between $[0, C_{max}]$ in a 2-dimensional space.

\paragraph{\texttt{Deceleration}}
Random $F_x$ and $F_y$ forces for the first $t_{acc}$ trajectory steps, and then $T - t_{acc}$ steps of deceleration, where $T$ is the time horizon. Deceleration at step $j>t_{acc}$ is done by setting $F_x^{j} = -\frac{1}{2}\frac{V_x^{j-1}}{\Delta t}, F_y^{j} = -\frac{1}{2}\frac{V_y^{j-1}}{\Delta t}$ (assuming the mass of the particle is $1$). $V$ and $\Delta t$ are defined in Section~\ref{ssec:domains_particle}. We use $\Delta_t = 0.1$.

\subsection{Reacher-v2}
\subsubsection{Datasets}
\label{reacher:datasets}
\paragraph{\texttt{FixedJoint}}
Trajectories were collected to represent a scenario where one of the two robot arm joints is malfunctioning and is force fixed in place. The policy can only control the other robotic arm joint. When evaluating policies, a validation set of 2,000 trajectories was used, which were unseen during training of the policies.

\subsection{Hopper-v2}
\subsubsection{Datasets}
\label{hopper:datasets}
\paragraph{Hopping}
The datasets of size 2180 trajectories used for sequence-lengths 64 and 128 were extracted from D4RL's \texttt{hopper-medium-v2}, and consist of mostly forward hopping behaviors. When evaluating policies, a validation set of 436 trajectories was used, which were unseen during training of the policies. Unlike in the other evaluated domains, where trajectories sampled from a random policy were used to train the VQ-VAE, in \texttt{Hopper-v2} we have used input videos from D4RL's \texttt{Hopper-medium-v2} - the reason is that using the initial random policy, the trajectories terminated (hopper fell down) before reaching the desired $T$. For \methodname\ training, we have modified \texttt{Hopper-v2} slightly so that the episode will not terminate when the Hopper falls, thus allowing it to reach T steps.

\subsection{GPT-Based Architecture}
\label{app:gpt}
The model is conditioned on the intent via cross-attention. The actor network is comprised of 2  hidden Linear layers of size 64, with a tanh activation.

The GPT model size hyper-parameters had an effect on the results, but these did not strictly improve with a bigger model. We experimented with several settings for number of heads (1, 2, 4) and number of layers (2, 4, 8) in the GPT model, and with these parameters, indeed the highest value gave better results. However, for the hidden dimension of the model, we tried 2 values (64, 128) and found the lower one gave better results. As for the context size, we only tried setting it to the entire trajectory length.

\subsection{GRU-Based Architecture}
\label{app:gru}
The single-layer GRU's hidden state size is set to match the flattened intent size of 4096. As with the GPT-based architecture, the actor network is comprised of 2 hidden layers of size 64 with tanh activations. 

\begin{table}[b]
    \vspace{-1em}
    \begin{minipage}[c]{0.5\textwidth}
    \centering
    \caption{RL hyperparameters}
    \label{table:rl_hps}
        \begin{tabular}{l|l}
        \hline
        \multicolumn{1}{c|}{Hyperparameter}     & \multicolumn{1}{c}{Value} \\ \hline
        PPO Clip Ratio                          & 0.2                       \\
        GAE $\lambda$                           & 0.95                      \\
        Discount rate $\gamma$                  & 0.99                      \\
        Learning rate                           & 1e-4                      \\
        Value loss coefficient                  & 0.5                       \\
        \# epochs                               & 4                         \\
        \# rollouts sampled per policy update   & 128                       \\
        Total iterations                        & 5000                     
        \end{tabular}
    \end{minipage}
    \begin{minipage}[c]{0.5\textwidth}
    \centering
    \caption{GAIL hyperparameters}
    \label{table:gail_hps}
        \begin{tabular}{l|l}
        \hline
        \multicolumn{1}{c|}{Hyperparameter}     & \multicolumn{1}{c}{Value} \\ \hline
        Discriminator batch size                & 128                       \\
        Discriminator learning rate             & 1e-3                      \\
        Gradient penalty $\lambda$              & 10                        \\
        \end{tabular}
    \end{minipage}
\end{table}

\subsection{RL Baseline}
\label{app:rl}
For the RL baseline, both actor and critic networks are 2-layer MLPs (of size 64) with tanh activations. Table~\ref{table:rl_hps} summarizes the hyperparameters used for training RL policies with PPO~\cite{schulman2017proximal} and a GRU policy. 

\subsubsection{MLP-Based Architecture for RL Baseline}
\label{app:rl_mlp}
For the RL baseline only, we also experimented with an MLP-based architecture. In this case, both the the intent and observation go through a single Linear layer followed by Layer Normalization and a ReLU activation. The output sizes of the Linear layers we used were 256 and 64 for the intents and observations, respectively. All other hyperparameters are identical to those shown in Table~\ref{table:rl_hps}.

\subsection{GAIL Baseline}
\label{app:gail}
Our GAIL from observations~\citep{torabi2018generative} implementation is based on~\citet{pytorchrl}, and uses PPO as the RL algorithm. We note that~\citet{pytorchrl} implements "vanilla" GAIL~\citep{ho2016generative}, which uses state + action pairs as input to the discriminator. Therefore, to match the setup of~\citet{torabi2018generative}, we modified the implementation so that the discriminator is fed state + next-state pairs. To add the intent as context to GAIL, it is concatenated to the state transition pair. Due to the large discrepancy in size between the intent (size 4096) and the state transition pair (size 8 in the case of \texttt{Particle:Splines}), before concatenating the intent to the state we downscale it to size 256 using a Linear layer. In addition to the downscaled intent, we also concatenate the timestep of the transition (normalized between $[-1,1]$) to the discriminator input. We found this improved GAIL performance in our experiments on \texttt{Particle:Splines}. The discriminator itself is a 3-layer MLP with a hidden dimension of 100 and tanh activations. GAIL specific hyperparameters are provided in Table~\ref{table:gail_hps}. PPO hyperparameters used in GAIL experiments are the same as in Table~\ref{table:rl_hps}.

We experimented with two reward modes involving GAIL: (1) \texttt{GAIL}: The standard formulation of GAIL, where the reward to the RL algorithm is the log of the discriminator output, and (2) \texttt{GAIL+STATE-MSE}, where we combined (by simple addition) the standard GAIL reward and the \texttt{STATE-MSE} reward defined in Section~\ref{ssec:rl_baselines}, in an attempt to see if a combination of the 2 reward signals (from the environment and from the discriminator) will result in an improved overall signal. As can be seen in Figure~\ref{fig:rl_appendix}, while (2) indeed improved over (1), neither was able to outperform PPO with \texttt{STATE-MSE} nor \methodname.

\newpage

\section{Additional Experimental Results}

\subsection{RL and IRL Baselines}
\label{app:rl_irl_baselines}
In Figure~\ref{fig:rl_appendix} we present additional results comparing \methodname\ to RL and IRL baselines.

As described in Section~\ref{ssec:rl_baselines}, \texttt{INTENT-MSE} is a sparse RL reward, calculated as the MSE between intents of the desired trajectory and the executed trajectory, given at the end of the episode. In our experiments, \texttt{INTENT-MSE} doesn't come close to \methodname, which we attribute to the more difficult learning from sparse reward. \\ 
We found the MLP-based RL policy (see Appendix~\ref{app:rl_mlp}) underperformed the GRU-based RL policy (though not too significantly). This can be attributed to the task being that of tracking a trajectory (specified by the intent), which can be easier with access to the trajectory performed by the agent so far. \\
Finally, Figure~\ref{fig:rl_appendix} shows the results for \texttt{GAIL+STATE-MSE}, the additional reward mode involving GAIL, as discussed in Appendix~\ref{app:gail}. This reward mode improves upon the standard GAIL reward, but does not perform better than neither \texttt{STATE-MSE} or \methodname.

In Figure~\ref{fig:training_curves} we present example training curves comparing \methodname\ and PPO with \texttt{STATE-MSE} reward.

\begin{figure}[H]
    \centering
    \vspace{-1em}
    \begin{minipage}{1\textwidth}
        \centering
        \includegraphics[width=0.7\columnwidth,keepaspectratio]{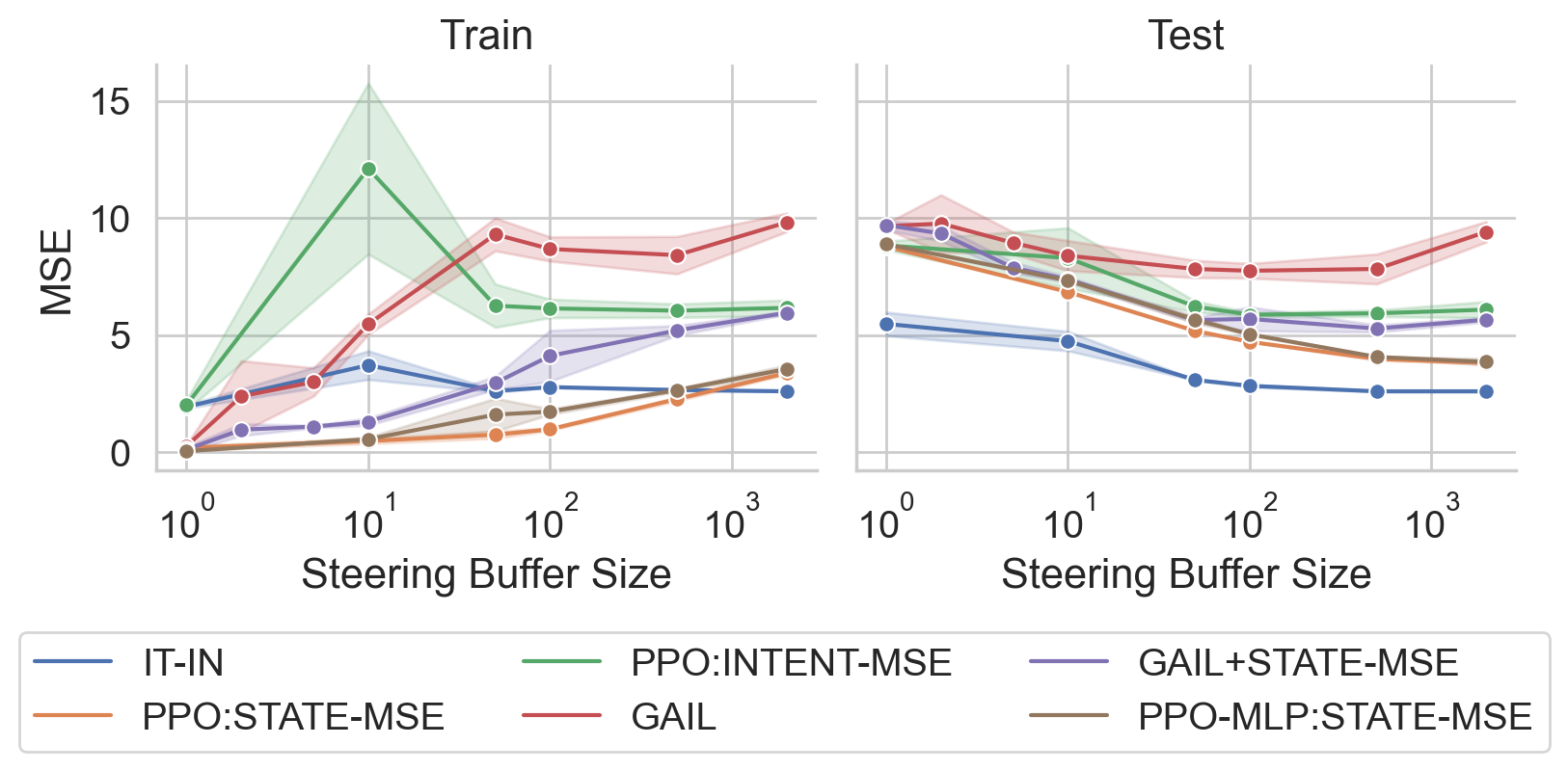}
        \caption{Comparison of \methodname\ with PPO and GAIL baselines. This figure includes the results shown in Figure~\ref{fig:rl} and additional results discussed in Appendix~\ref{app:rl_irl_baselines}. All experiments are on the \texttt{Particle:Splines} environment with $T=16$. All results are MSE (lower is better), each represented with a mean and standard deviation of 3 random seeds. Note that \methodname\ outperforms baselines on test trajectories (graph on the right) for all $\steeringbuffer$ sizes. For GAIL, $\steeringbuffer$ is also used as the expert data for the discriminator. }
        \label{fig:rl_appendix}
    \end{minipage}
    \begin{minipage}{1\textwidth}
        \centering
        \includegraphics[scale=0.5]{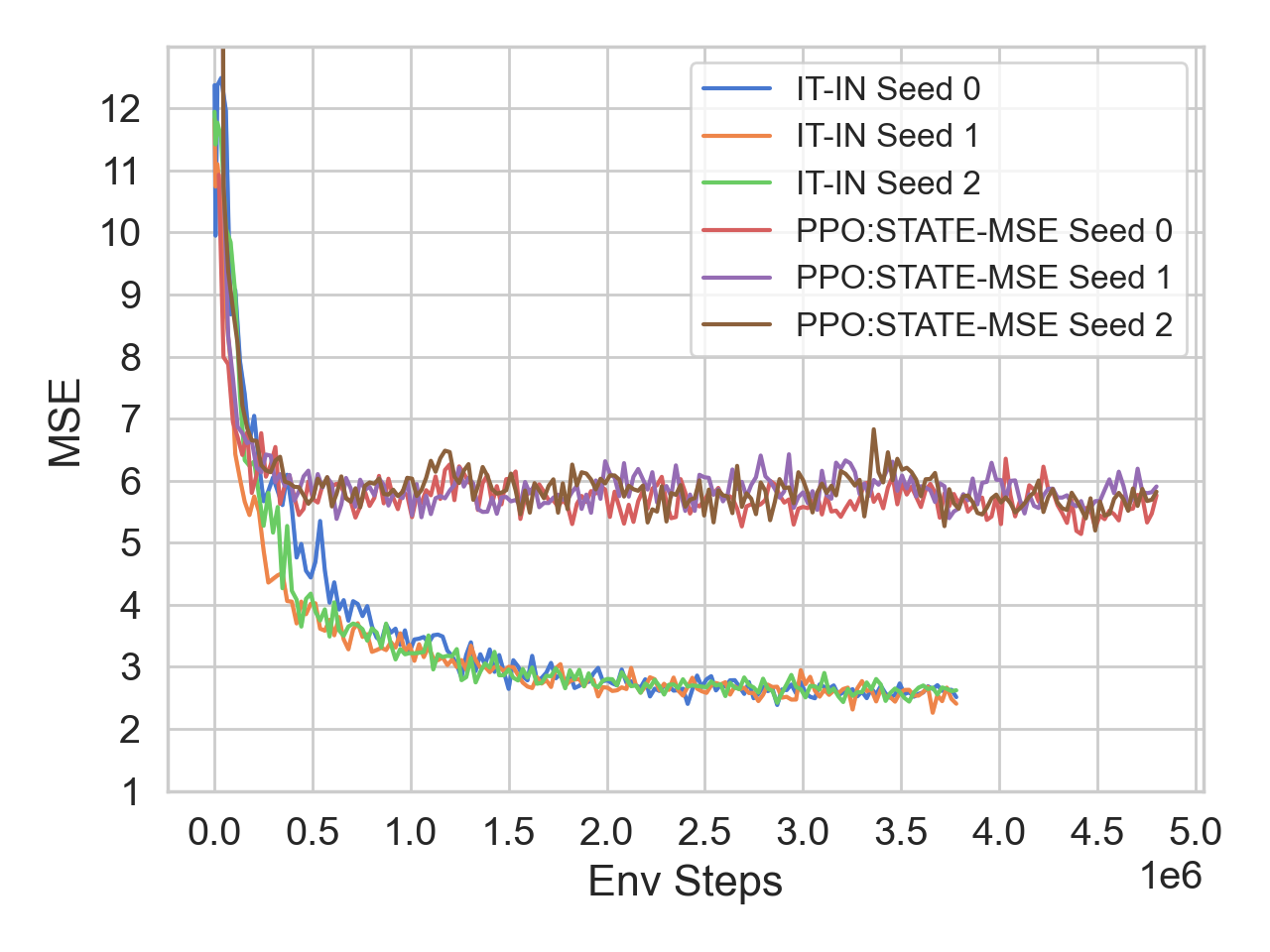}
        \caption{Comparison of test MSE convergence of \methodname\ vs. PPO with \texttt{STATE-MSE} reward. All runs are in the \texttt{Particle:Splines} environment, with $|\steeringbuffer|=500$ and horizon $T=16$. The MSE is calculated on a held out set of 500 trajectories. }
        \label{fig:training_curves}
    \end{minipage}
\end{figure}

\subsection{\texttt{Particle:Splines} - Effect of Trajectory Length}
\label{app:results_splines_traj_length}
We tested \methodname\ on multiple horizons $T$ in the \texttt{Splines} domain, and found it to work well across horizons of 32, 64 and 128. We present sample visualizations with different $T$ values in Figure~\ref{fig:particle_splines_multi_lengths_grid} (showing the final reconstructed trajectories) and in Figure~\ref{fig:splines_multi_length_strips} (showing trajectory progression during an epsiode).

\subsection{\texttt{Particle:Splines} - Effect of Steering Dataset Size}
\label{app:results_splines_steering_size}
In Figure~\ref{fig:splines_steering_grid} we present trajectory visualizations showcasing the effect of the size of the steering buffer $\steeringbuffer$ (cf. Table~\ref{table:MSEs}).

\subsection{\texttt{Particle:Deceleration} - Effect of Steering Dataset Size}
\label{app:results_deceleration}
Similarly to Section~\ref{app:results_splines_steering_size}, in Figure~\ref{fig:deceleration_steering_grid}, we showcase the effect of the size of the steering buffer $\steeringbuffer$ (cf. Table~\ref{table:MSEs}) in the \texttt{Particle:Deceleration} domain.

\subsection{\texttt{Reacher-v2}}
We present sample reconstruction visualizations for random-action trajectories from \texttt{Reacher-v2} on 16-step sequences in Figure~\ref{fig:reacher_random_to_random}. Sample videos for 64-step \texttt{FixedJoint} sequences (trained with a GPT-based policy) can be found in the project's website: \url{https://sites.google.com/view/iter-inver}.

\subsection{\texttt{Hopper-v2}}
\label{app:results_hopper}
We show additional examples of rollouts for the \texttt{Hopper-v2} domain on 128-step sequences in Figure~\ref{fig:hopper_128_appendix}.

\subsection{Exploration}
Table~\ref{table:test_exploration} is showcasing \texttt{Splines} and \texttt{Hopper-v2} reconstruction MSEs when trained with and without exploration noise $\noise$.

\begin{table}[]
    \centering
    \caption{Evaluation of policies trained with and without exploration. We show average MSE for 3 policies; due to different domains, MSEs are comparable only within each row.  
    }
    \label{table:test_exploration}
    \begin{tabular}{ccc}
        \toprule
        \multicolumn{1}{c}{}  & Exploration Noise $\noise$ & No Exploration \\
        \midrule
        \texttt{Particle:Splines}, $T=64$, $|\steeringbuffer| = 500$  & 69.2  & 454.7 \\
        \texttt{Hopper-v2}, $T=128$, $|\steeringbuffer| = 1740$        & 483.3 & 920.6 \\
        \bottomrule
    \end{tabular}%
\end{table}

\subsection{Steering Cross-Evaluation}
\label{app:test_on_ood}
In Figure~\ref{fig:drive_towards_desired} we show example rollouts from the experiments on steering cross-evaluation (cf. Table~\ref{table:steering}).

\subsection{GRU-Based Policy Experiments}
We report similar results with a GRU-based policy to the the results shown in Table~\ref{table:MSEs} (analyzing the effect of steering dataset size) in Table~\ref{table:GRU_steering}, and similar results to Table~\ref{table:steering} (steering cross-evaluation) in Table~\ref{table:GRU_ood}.

\begin{table}[]
    \centering
    \caption{
    Evaluation of \methodname\ with a GRU policy on variable Steering Dataset size. $T=16$.
    Note that $|\steeringbuffer|=0$ represents the case where no steering is used at all. In this case, we use trajectories sampled from a random policy to initialize $|\prevbuffer|$ (see Algorithm~\ref{alg:ESI}). Note: Since we do not normalize the MSE w.r.t. $T$, these results have a different scale than Table~\ref{table:MSEs}. }
    \label{table:GRU_steering}
    \resizebox{\textwidth}{!}{%
    \begin{tabular}{lcccccc}
        \toprule
        \multicolumn{1}{c}{}  & $|\steeringbuffer|$ = 0 & $|\steeringbuffer|$ = 10 & $|\steeringbuffer|$ = 50 & $|\steeringbuffer|$ = 100 & $|\steeringbuffer|$ = 500 & $|\steeringbuffer|$ = 5000 \\
        \midrule
        \texttt{Particle:Splines}       & 5.48 & 5.18 & 3.86 & 3.61 & 3.02 & 2.89 \\
        \texttt{Particle:Deceleration}  & 0.85 & 0.89 & 0.75 & 0.73 & 0.67 & 0.71 \\
        \texttt{Reacher-v2:FixedJoint}  & 2.49 & 2.05 & 1.68 & 1.64 & 1.58 & 1.61 \\
        \bottomrule
    \end{tabular}%
    }
\end{table}

\begin{table}[]
    \centering
    \caption{Steering cross-evaluation for a GRU-policy. Horizon $T=16$. In all cases $|\steeringbuffer|=500$.}
    \label{table:GRU_ood}
    \begin{tabular}{ccc}
        \toprule
        Test Dataset & Steering Dataset & MSE \\
        \midrule
        \multirow{2}{*}{\texttt{Splines}}      & \texttt{Splines}      & 2.79  \\
                                               & \texttt{Deceleration} & 5.4 \\
        \midrule
        \multirow{2}{*}{\texttt{Deceleration}} & \texttt{Splines}      & 1.46  \\
                                               & \texttt{Deceleration} & 0.72  \\
        \bottomrule
    \end{tabular}
\end{table}

\subsection{Non-Deterministic Dynamics}
\label{app:non_deterministic_dynamics}

In this section we discuss the effect of non-deterministic dynamics on the performance of our method. We consider two sources of non-determinism: non-deterministic starting states and non-deterministic transitions.

For non-deterministic starting states, we can look at our results in the \texttt{Hopper-v2} environment as an example (see Table~\ref{table:MSEs} and Figures~\ref{fig:hopper_128} and~\ref{fig:hopper_128_appendix}). In this environment the initial state is randomized, albeit from a rather limited distribution (uniform noise of scale 5e-3 is added to initial position and velocity), and this doesn't pose a problem for our method. That said, for an initial state distribution that is very diverse, it is not clear how the policy could recover the trajectory specified by an intent with a very different starting state. This limitation is therefore inherent to the way we defined the problem (trajectory tracking).

As for non-deterministic transitions, we evaluated a variant of the \texttt{Particle:Splines} environment with noisy transitions: zero-mean Gaussian noise with standard deviation $\sigma$ is added to the action (see Figure~\ref{fig:noisy_particle_env} for examples of trajectories with different noise scales added). Figure~\ref{fig:IT-IN_perf_on_noisy_particle} shows that for moderate $\sigma$ our method still works well, while for $\sigma > 5$ performance starts to deteriorate considerably.

Thus, we find empirically that our method can handle some stochasticity in the transitions and in the starting state. We defer a more complete characterization, including the non-trivial theoretical analysis, to future work.

\begin{figure}[]
    \centering
    \includegraphics[width=0.95\textwidth,keepaspectratio]{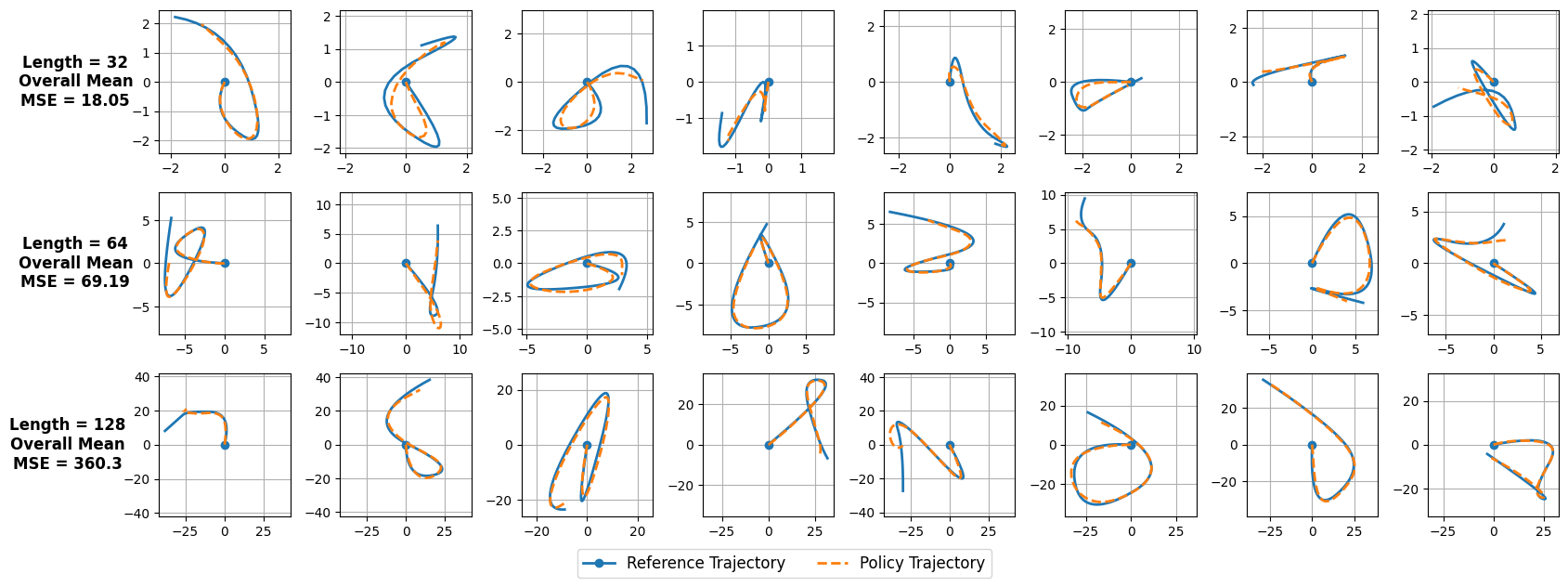}
    \caption{Example results on the \texttt{Splines} dataset, for different sequence lengths. In all cases shown here $|\steeringbuffer|=500$. To the left of each row we state the average MSE on an evaluation set of 3 policies trained with different seeds. Note the increasing scale of the plots as the sequence length increases. Also note that all trajectories start at \texttt{(0,0)}, marked by the blue circle in each plot. }
    \label{fig:particle_splines_multi_lengths_grid}

    \vspace{1em}
    \begin{subfigure}{1\textwidth}
        \centering
        \includegraphics[width=\textwidth,keepaspectratio]{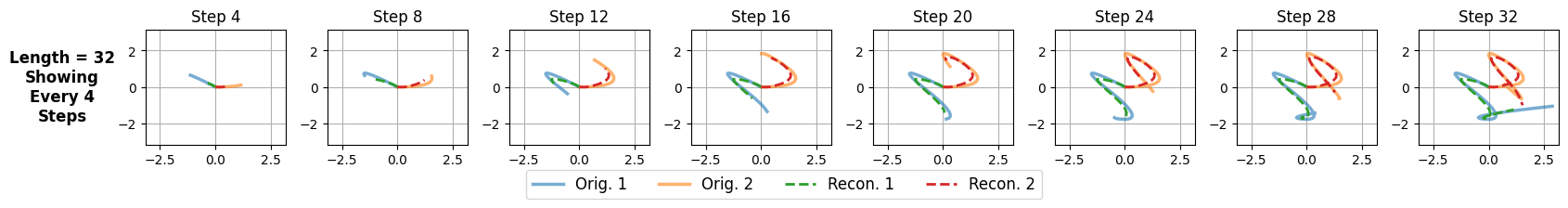}
    \end{subfigure}
    \begin{subfigure}{1\textwidth}
        \centering
        \includegraphics[width=\textwidth,keepaspectratio]{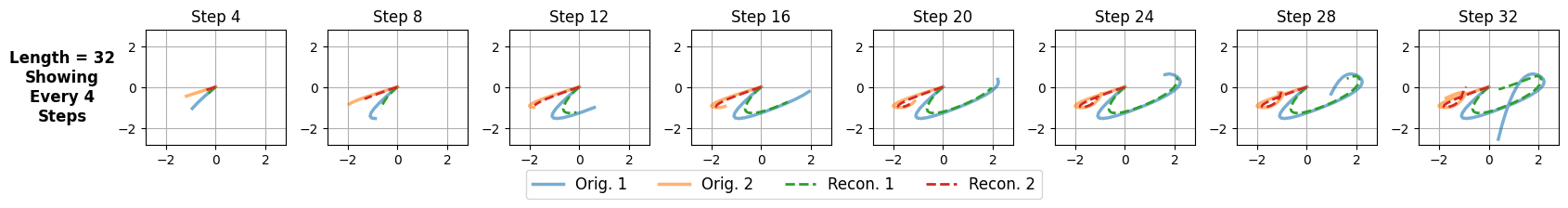}
    \end{subfigure}
    \begin{subfigure}{1\textwidth}
        \centering
        \includegraphics[width=\textwidth,keepaspectratio]{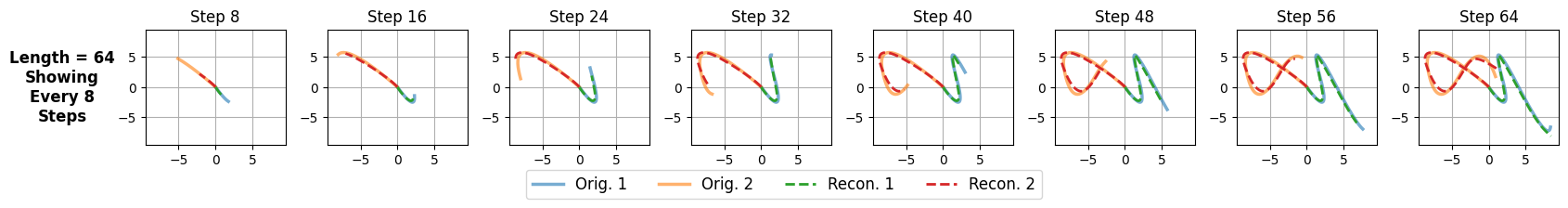}
    \end{subfigure}
    \begin{subfigure}{1\textwidth}
        \centering
        \includegraphics[width=\textwidth,keepaspectratio]{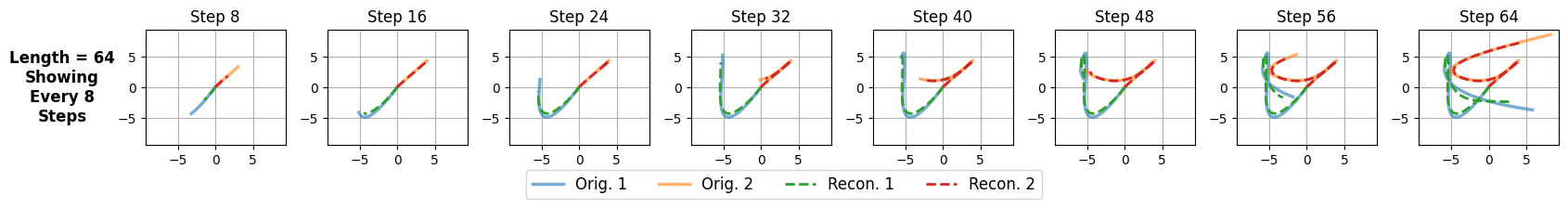}
    \end{subfigure}
    \begin{subfigure}{1\textwidth}
        \centering
        \includegraphics[width=\textwidth,keepaspectratio]{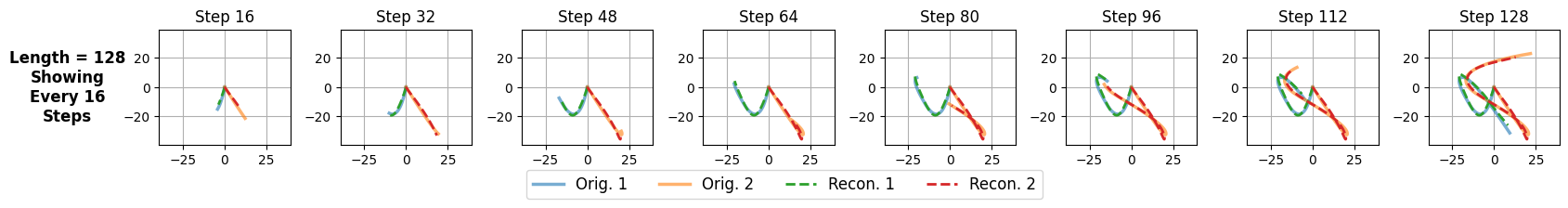}
    \end{subfigure}
    \begin{subfigure}{1\textwidth}
        \centering
        \includegraphics[width=\textwidth,keepaspectratio]{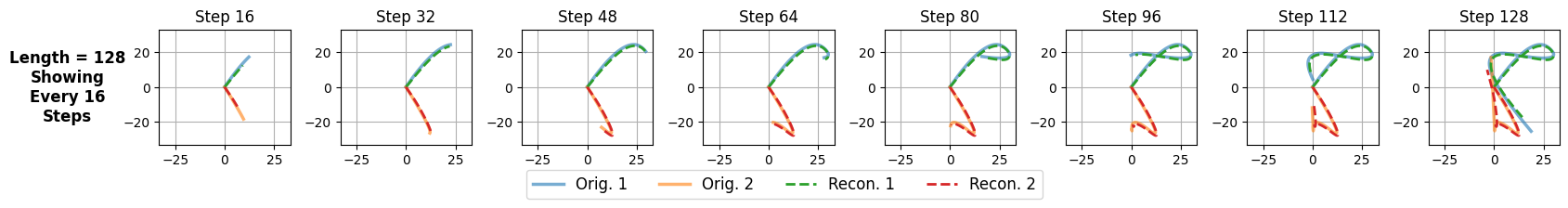}
    \end{subfigure}
    \caption{Visualization of trajectories progression in the \texttt{Splines} domain for different horizons $T$. Note the increasing scale of the plots as the sequence length increases. $|\steeringbuffer| = 500$.
    }
    \label{fig:splines_multi_length_strips}
\end{figure}

\begin{figure}
    \centering
    \includegraphics[width=\textwidth,keepaspectratio]{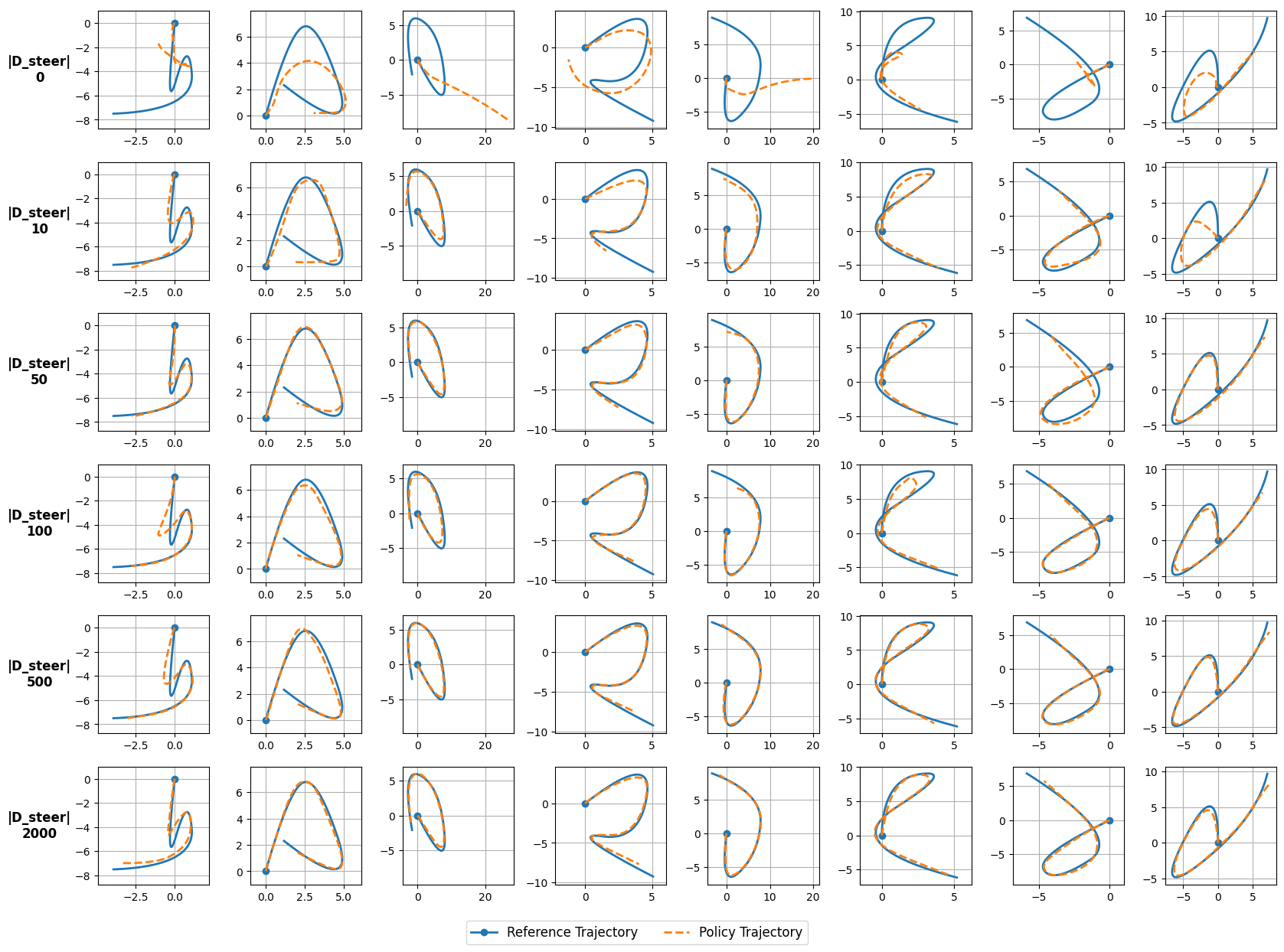}
    \vspace{-1em}
    \caption{Example results on the \texttt{Particle:Splines} dataset for policies trained with different sizes of $\steeringbuffer$.  Each row corresponds to a different size. Each column corresponds to a specific reference trajectory from the dataset, the intent of which was used as input to the policies. $T=64$ was used in all experiments.}
    \label{fig:splines_steering_grid}
\end{figure}

\begin{figure}
    \centering
    \includegraphics[width=\textwidth,keepaspectratio]{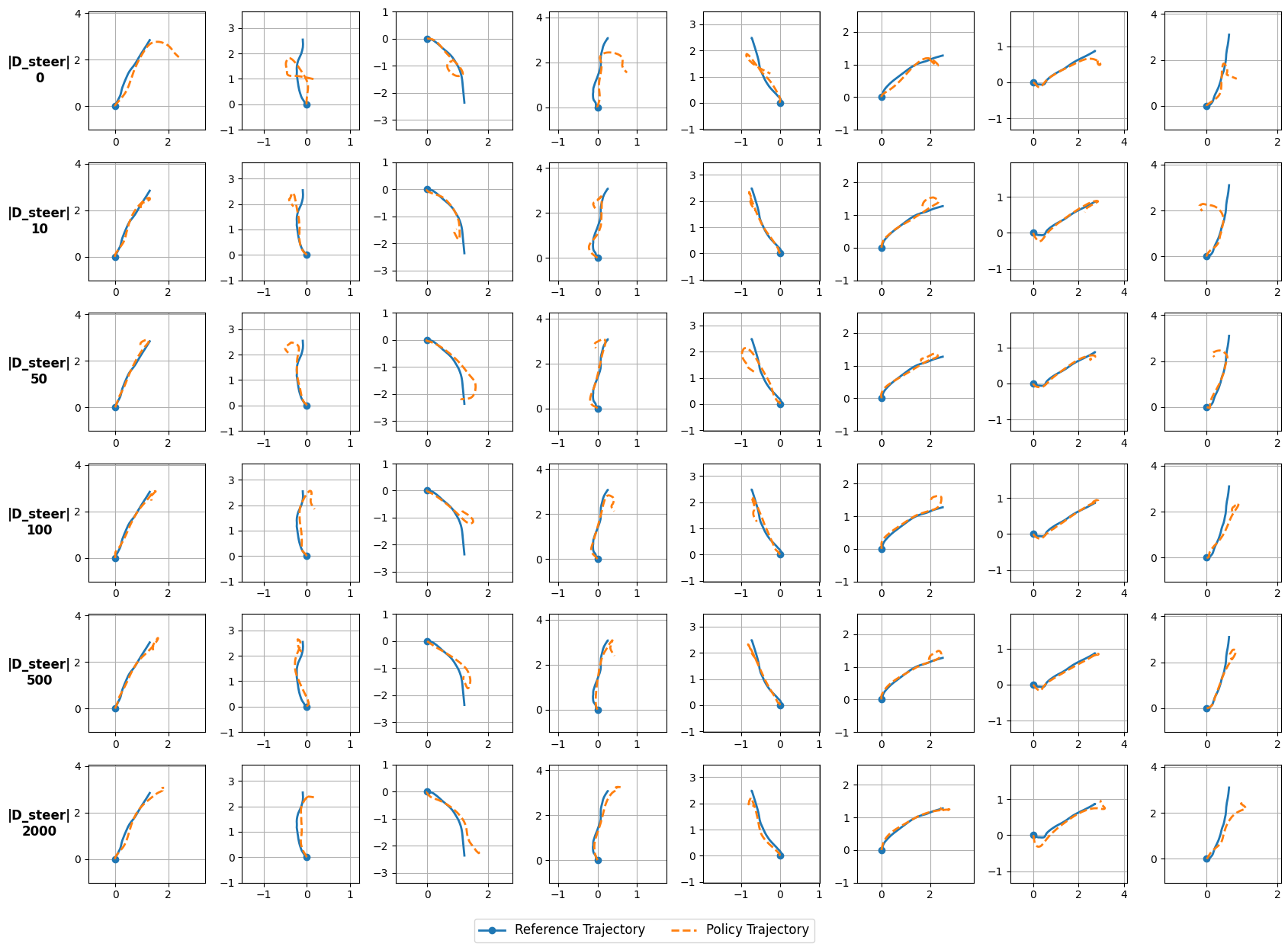}
    \vspace{-1em}
    \caption{Example results on the \texttt{Particle:Deceleration} dataset for policies trained with different sizes of $\steeringbuffer$. $T=64$ was used in all experiments. Figure structure same as in Figure~\ref{fig:splines_steering_grid}. }
    \label{fig:deceleration_steering_grid}
\end{figure}

\begin{figure}
    \centering
    \includegraphics[width=\textwidth,keepaspectratio]{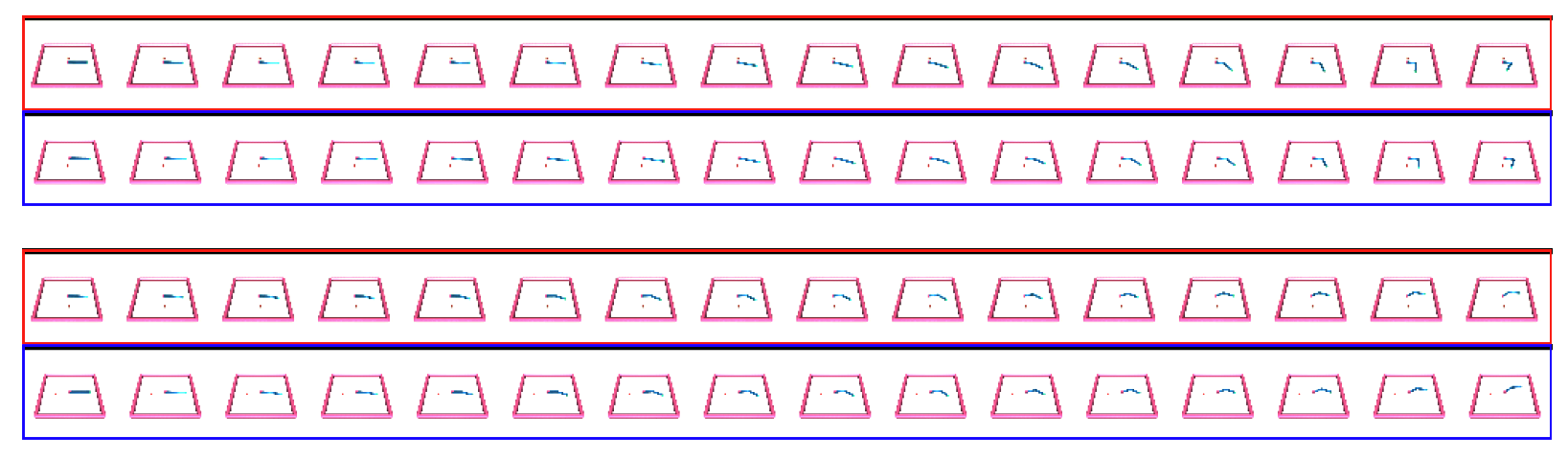}
    \caption{Examples of trajectory reconstructions in the Reacher-v2 domain. In each plot, the red row is the reference trajectory and the blue row is the policy reconstruction. These are based on a GRU policy.
    For ease of viewing, we modified the dark colors of the original rendered images. }
    \label{fig:reacher_random_to_random}
\end{figure}

\begin{figure}
    \centering
    \includegraphics[width=\textwidth,keepaspectratio]{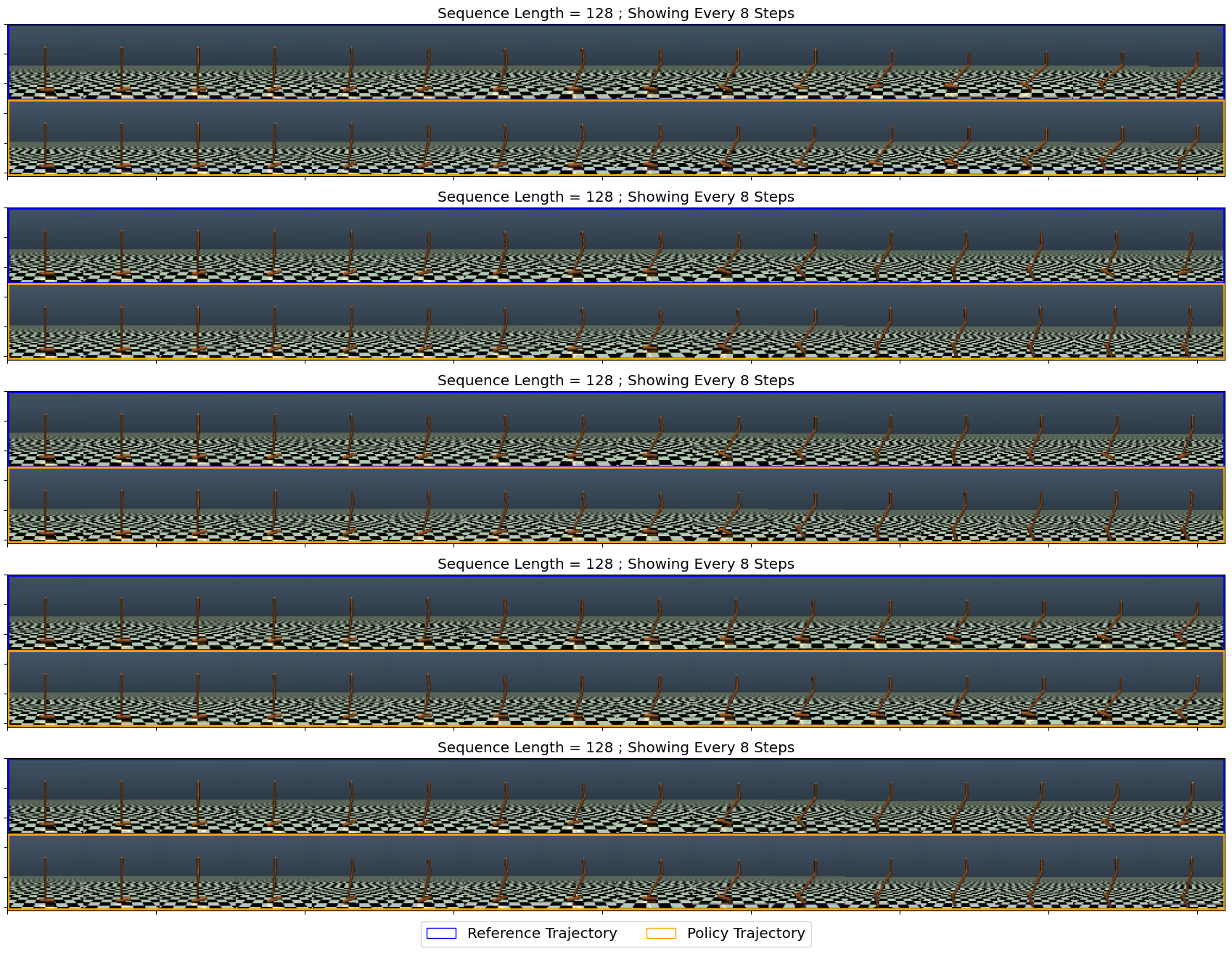}
    \vspace{-2em}
    \caption{Examples of trajectory reconstructions in the \texttt{Hopper-v2} domain, with $T=128$ and $|\steeringbuffer|=500$. }
    \label{fig:hopper_128_appendix}
    \vspace{-1em}
\end{figure}

\begin{figure}
    \centering
    \begin{subfigure}{1\textwidth}
        \centering
        \includegraphics[width=\textwidth,keepaspectratio]{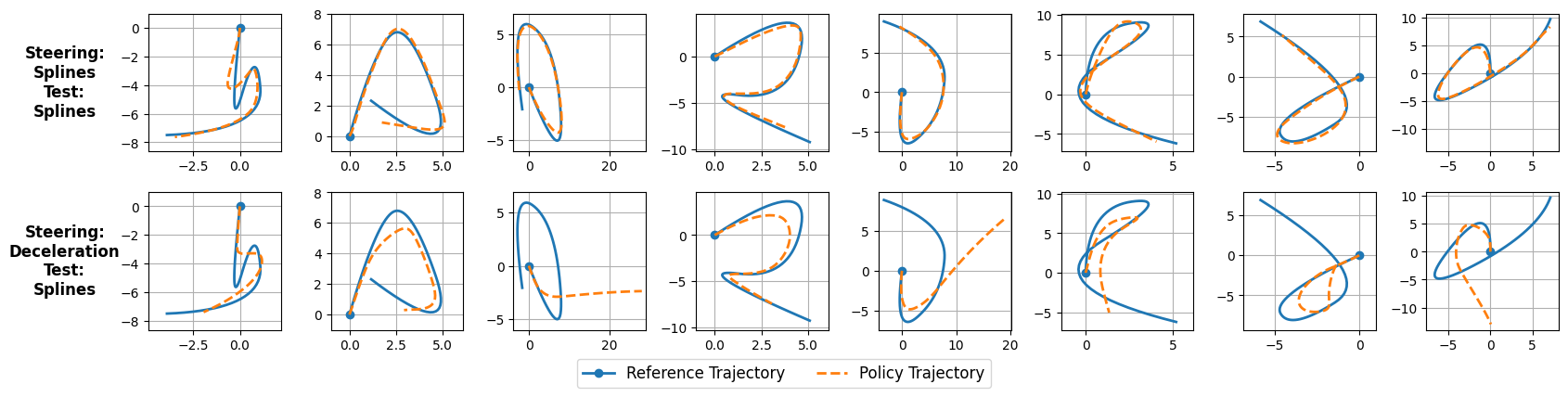}
        \caption{Testing on trajectories from \texttt{Particle:Splines}}
        \label{fig:drive_towards_desired_splines}
    \end{subfigure}
    \begin{subfigure}{1\textwidth}
        \centering
        \includegraphics[width=\textwidth,keepaspectratio]{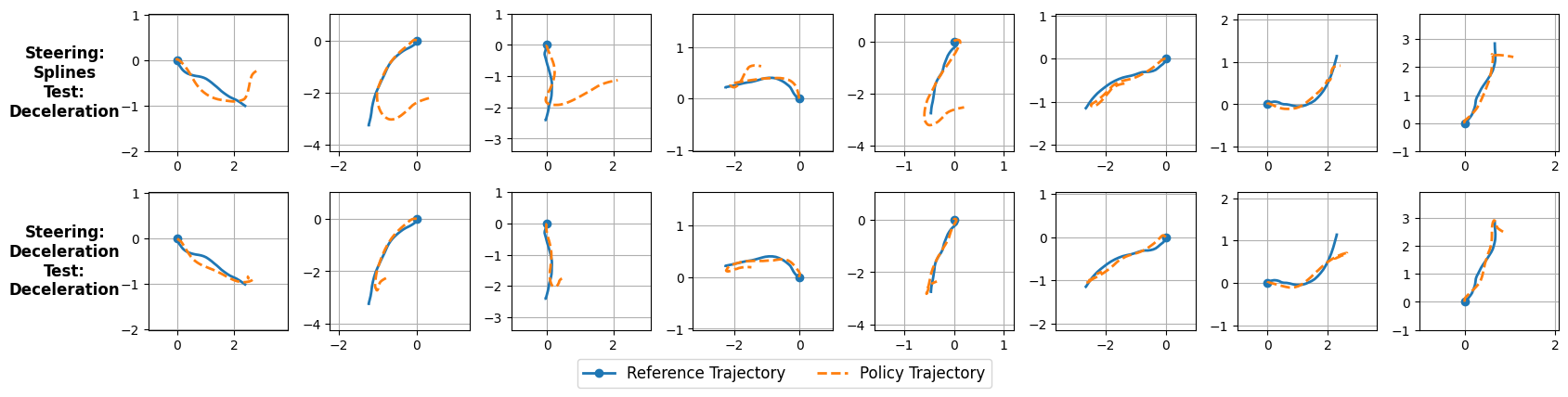}
        \caption{Testing on trajectories from \texttt{Particle:Deceleration}}
        \label{fig:drive_towards_desired_deceleration}
    \end{subfigure}
    \caption{Examples comparing how policies trained with steering intents from either \texttt{Particle:Splines} or \texttt{Particle:Deceleration} perform when tested on trajectories from either datasets. We can see that when a policy is trained with steering intents from one dataset, it performs well on that dataset and performs poorly on the other. In each column the reference trajectory is the same. }
    \label{fig:drive_towards_desired}
\end{figure}

\begin{figure}[]
    \centering
    \begin{subfigure}{0.45\textwidth}
        \centering
        \includegraphics[width=\textwidth,keepaspectratio]{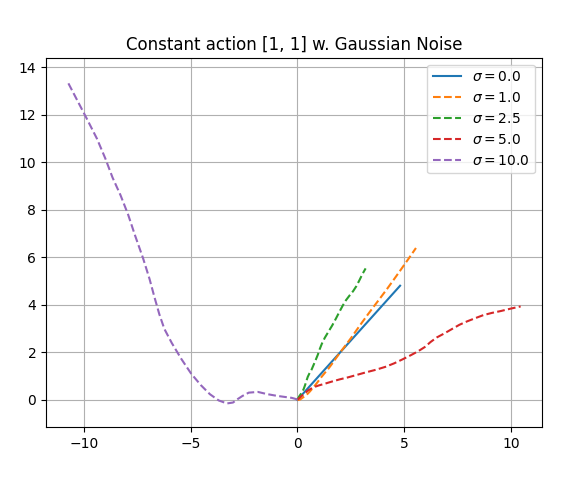}
        \caption{The \texttt{Particle} environment ($T=32$) dynamics are altered to include a zero-mean Gaussian noise with standard deviation $\sigma$ added to action. This figure shows how a fixed acceleration trajectory [action = (1, 1), for 32 consecutive steps] is affected by different noise scales. }
        \label{fig:noisy_particle_env}
    \end{subfigure}
    \hspace{2em}
    \begin{subfigure}{0.45\textwidth}
        \centering
        \includegraphics[width=\textwidth,keepaspectratio]{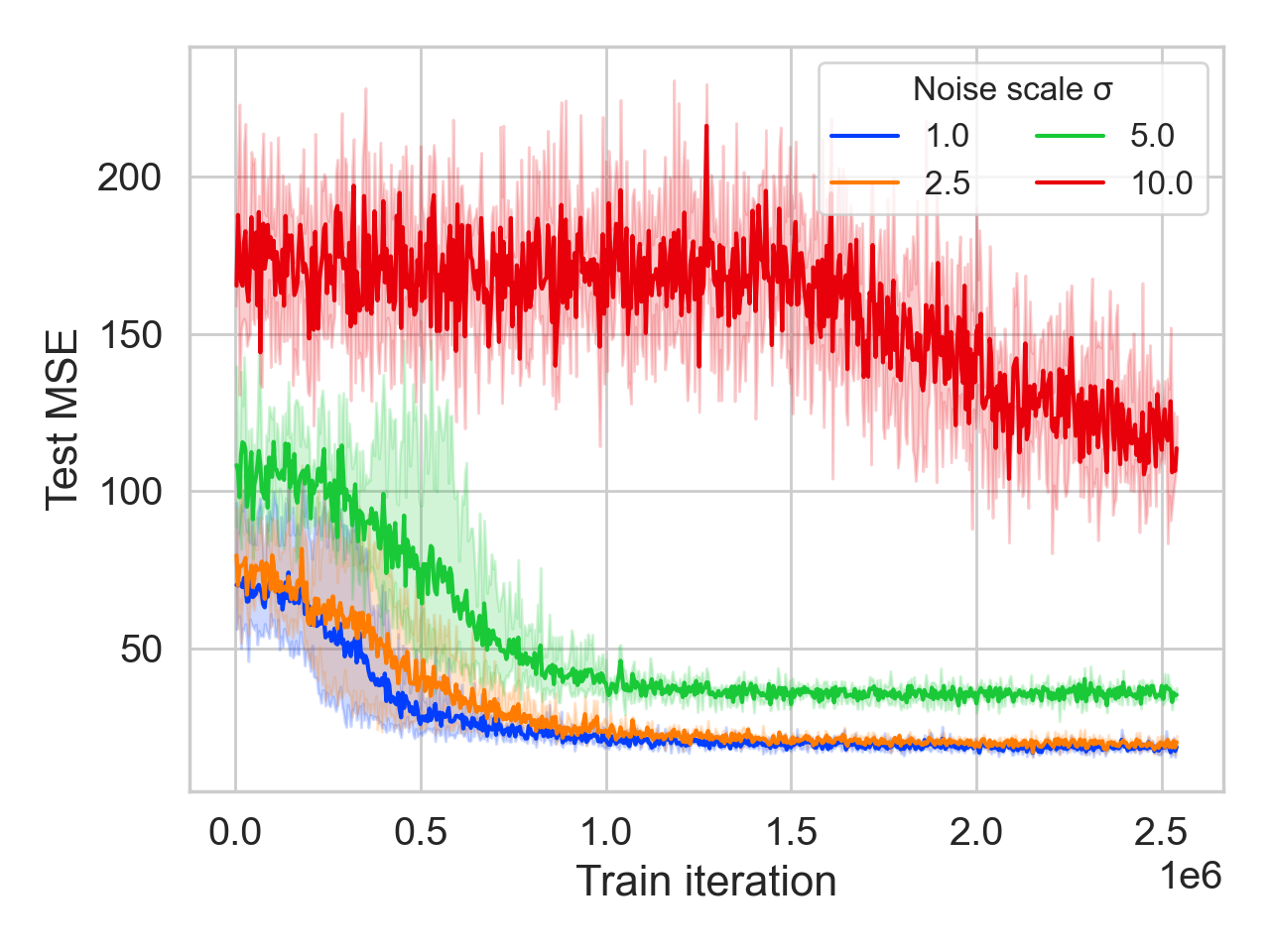}
        \caption{Test MSE on a test set of 500 trajectories not included in the steering dataset, in the \texttt{Particle:Splines} environment with noise added to the actions. For all runs $T=32$ and $|\steeringbuffer|=500$. X-axis is number of training iterations.  Each $\sigma$ value was evaluated on 3 seeds. For moderate noise scales ($\sigma = 1$ and $\sigma = 2.5$) our method still works well (MSE = 18.5 and 20.7 respectively vs. MSE = 18.0 for no noise added - see top row in Figure~\ref{fig:particle_splines_multi_lengths_grid}), while for larger noise scales ($\sigma = 5.0$ and $\sigma = 10.0$) the performance starts to deteriorate considerably. }
        \label{fig:IT-IN_perf_on_noisy_particle}
    \end{subfigure}
    \caption{Analysis of non-deterministic dynamics (see Appendix~\ref{app:non_deterministic_dynamics}) }
\end{figure}

\end{document}